\documentclass[lettersize,journal]{IEEEtran}

\usepackage{orcidlink}
\hypersetup{hidelinks} % hide the hyperlink box
\usepackage{hyperref}
\usepackage{amsmath,amsfonts}
\usepackage{diagbox}
\usepackage[capitalize,noabbrev]{cleveref}
\usepackage{amssymb}

\usepackage{mathtools}
\usepackage{multirow}
\usepackage{float}
\usepackage{graphicx}
\usepackage[figuresright]{rotating}
\usepackage{multicol}
\usepackage{paralist}
\usepackage{bm}
\usepackage{multirow}
\usepackage{color}
\usepackage{theorem}
\newtheorem{theorem}{Theorem}
\newtheorem{lemma}{Lemma}
\newtheorem{definition}{Definition}
\newtheorem{proof}{Proof}
\usepackage{algorithm}
\usepackage{graphicx}
\usepackage{algorithmic}

\usepackage{ifpdf}

\usepackage{booktabs}

\usepackage{array}
\usepackage[caption=false,font=normalsize,labelfont=sf,textfont=sf]{subfig}
\usepackage{textcomp}
\usepackage{stfloats}
\usepackage{url}
\usepackage{verbatim}

\def\BibTeX{{\rm B\kern-.05em{\sc i\kern-.025em b}\kern-.08em
    T\kern-.1667em\lower.7ex\hbox{E}\kern-.125emX}}
\usepackage{balance}

\begin{document}

\pdfoutput=1
\ifpdf
\title{Neural Operator Variational Inference based on Regularized Stein Discrepancy for Deep Gaussian Processes }
\author{Jian Xu\orcidlink{0000-0002-0350-5528}, Shian Du\orcidlink{0000-0001-9392-6846}, Junmei Yang\orcidlink{0000-0002-9677-0768}, Qianli Ma\orcidlink{0000-0002-9356-2883}, ~\IEEEmembership{Member,~IEEE}, \\ and Delu Zeng\orcidlink{0000-0001-7322-1873}, ~\IEEEmembership{Member,~IEEE}
\thanks{ Manuscript received April 10, 2023; revised December 17, 2023 and
February 29, 2024; accepted May 22, 2024. The work is supported by the Fundamental Research Program of Guangdong, China, under Grant 2023A1515011281; and in part by the National Natural Science Foundation of China under Grant 61571005. (\emph{Jian Xu and Shian Du contributed equally to this work.}) (\emph{Corresponding authors: Qianli Ma and Delu Zeng.})}
\thanks{Jian Xu (e-mail: 2713091379@qq.com) is with School
of Mathematics, South China University of Technology, Guangzhou, China. Shian Du (e-mail: dsa1458470007@gmail.com) is with Shenzhen International Graduate School, Tsinghua University, Shenzhen, China.}
\thanks{Junmei Yang (e-mail: yjunmei@scut.edu.cn), Delu Zeng (e-mail: dlzeng@scut.edu.cn) are affiliated with  School of Electronic and Information Engineering at South China University of Technology, Guangzhou, China. Qianli Ma (e-mail: qianlima@scut.edu.cn) is affiliated with  School of Computer Science  and Engineering at South China University of Technology, Guangzhou, China.}}

\maketitle

\begin{abstract}
Deep Gaussian Process (DGP) models offer a powerful nonparametric approach for Bayesian inference, but exact inference is typically intractable, motivating the use of various approximations. However, existing approaches, such as mean-field Gaussian assumptions, limit the expressiveness and efficacy of DGP models, while stochastic approximation can be computationally expensive. To tackle these challenges, we introduce Neural Operator Variational Inference (NOVI) for Deep Gaussian Processes. NOVI uses a neural generator to obtain a sampler and minimizes the Regularized Stein Discrepancy (RSD)  between  the generated distribution  and true posterior in $\mathcal{L}_2$ space. We solve the minimax problem using Monte Carlo estimation and subsampling stochastic optimization techniques, and demonstrate that the bias introduced by our method can be controlled by multiplying the Fisher divergence with a constant, which leads to robust error control and ensures the stability and precision of the algorithm. Our experiments on datasets ranging from hundreds to millions demonstrate the effectiveness and the faster convergence rate of the proposed method. We achieve a classification accuracy of 93.56 on the CIFAR10 dataset, outperforming state-of-the-art (SOTA) Gaussian process methods. We are optimistic that NOVI possesses the potential to enhance the performance of deep Bayesian nonparametric models and could have significant implications for various practical applications.
\end{abstract}


\section{Introduction}
\label{sec: introduction}
\IEEEPARstart{G}{aussian} 
 processes (GPs) \cite{Rasmussen06} are widely used in statistical inference and machine learning due to their effectiveness in modeling the relationship between inputs and outputs. For example, they have been successfully applied to modeling the dynamics of complex systems, such as robots or autonomous vehicles, for tasks such as trajectory planning \cite{cheng2022real}, adaptive control \cite{chowdhary2014bayesian}, and anomaly detection \cite{cho2019hierarchical}. The assumption that the latent function values follow a joint Gaussian distribution may not always hold, and in some scenarios, it can be overly restrictive \cite{dutordoir2021deep}. For example, when dealing with non-Gaussian and non-stationary processes \cite{hebbal2018efficient,lu2020interpretable}, such as those found in financial time series or climate modeling, the Gaussian assumption may not be appropriate. Therefore, Deep Gaussian processes (DGPs) have been proposed as an alternative approach to address these limitations in GP models.

A Deep Gaussian Process (DGP) model is a hierarchical composition of GP models that offers a probabilistic nonparametric approach with robust uncertainty quantification \cite{ober2021global}. The non-Gaussian distribution over composition functions provides expressive capacity but also presents challenges for inference \cite{dunlop2018deep}. Previous research on DGP models has used variational inference with a combination of sparse Gaussian processes \cite{Snelson06,candela05,gibbs2000variational,mao2020multiview} and mean-field Gaussian assumptions \cite{pmlr-v38-hensman15,deisenroth2015distributed,Yarin14,Lawrence13,NghiaICML16,HoangICML16,Titsias09a,liu2020gaussian} to approximate the posterior distribution.Stochastic optimization techniques have been used to scale up DGPs to handle large datasets, such as Doubly Stochastic Variational Inference (DSVI) 
\cite{salimbeni2017doubly}.These strategies often incorporate a collection of inducing points ($M\ll N$) whose position is learned alongside the other model hyperparameters, reduicng the training cost to $\mathcal{O} \left( NM^2 \right)$.

The mean-field Gaussian assumptions in approximate posterior distributions simplify computations, but can impose overly stringent constraints on DGP models, potentially limiting their expressiveness and effectiveness. Stochastic approximation approaches, such as SGHMC \cite{havasi2018inference}, draw unbiased samples from the posterior distribution, but their sequential sampling method can be computationally expensive for both training and prediction. Additionally, evaluating their convergence in finite time can be challenging \cite{gao2021global}.

While previous literature has explored various methods to approximate the non-mean-field Gaussian posterior, to the best of our knowledge, none has fully addressed the important problem of inducing point distribution in DGP inference. For instance, previous approaches, such as those proposed in \cite{wu2021hierarchical, lindinger2020beyond, shi2020sparse}, attempted to design grids, orthogonal structures, or other special structures among inducing points. However, such structures may be handcrafted and may introduce bias by not fully capturing the information of the inducing points from the data. Other approaches, such as those based on normalizing flows \cite{rezende2015variational, yu2021convolutional}, face invertibility constraints that limit the flexibility of the transformation form of neural networks \cite{dinh2016density}. In addition, implicit distributions variational inference  \cite{mescheder2017adversarial, ma2019variational, yu2019implicit, sun2019functional, rodriguez2022adversarial, rodriguez2022function} attempts to estimate the difficult-to-handle non-Gaussian posterior variational lower bound through adversarial networks. However, controlling the bias and variance of the density ratio in high-dimensional space becomes exceedingly difficult, hindering the scalability and effectiveness of this approach \cite{sugiyama2012density, titsias2019unbiased}.

Therefore, in the context of high-dimensional non-mean-field scenarios, we propose a new variational inference framework for DGP based on Stein discrepancy (SD) \cite{liu2016stein,wang2020stabilizing}. This is because SD provides accurate and efficient measures of distance between probability distributions, alleviating various computational issues in computing KL divergence. Unlike KL divergence, SD does not require computing the normalization constant of the distribution, and its gradient form often contains high-order information and geometric properties such as the curvature of the distribution, making it more effective for optimizing DGP high-dimensional non-mean-field posterior distributions and providing advantages in terms of error control and convergence rate \cite{goodfellow2016deep,huggins2018practical}.

In this work, we introduce a novel inference framework for DGP models called Neural Operator Variational Inference (NOVI), which utilizes operators to optimize a regularized Stein Discrepancy with data subsampling.Specifically, we use Gaussian noise to transform a simple low-dimensional distribution into a high-dimensional complex distribution through a neural network generator, and then minimize the regularized Stein discrepancy  between the generated distribution and the true posterior distribution using a SGD-based approach to obtain the gradients of the generator. The NOVI approach solves a minimax problem by Monte Carlo estimation, and offers a black-box algorithmic solution that can handle complex posterior distributions for DGP models.

The main contributions are as follows:
\begin{itemize}

\item[$\bullet$] We propose NOVI for DGPs, a novel variational framework based on Stein discrepancy and operator variational inference with a neural generator. It minimizes Regularized Stein Discrepancy in $\mathcal{L}_2$ space between the generated  distribution and true posterior to construct a more flexible and wider class of posterior approximations overcoming previous limitations caused by mean-field Gaussian posterior assumptions and the issues of
non-mean-field minimization of KL divergence in  the context of high-dimensional inducing points scenarios.
\item[$\bullet$] We provide theoretical evidence that our training schedule is essentially optimizing the Fisher divergence between  the generated distribution and the true posterior distribution. Additionally, the bias introduced by our method can be effectively controlled by multiplying the Fisher divergence with a constant. This feature of our approach enables us to achieve robust error control, ensuring the stability and precision of the algorithm.
\item[$\bullet$] We have conducted experimental demonstrations on eight UCI regression datasets and image classification datasets, which include MNIST, Fashion-MNIST, and CIFAR-10. The results demonstrate remarkable performance and faster convergence speeds than state-of-the-art methods, validating the effectiveness of the proposed model. By employing a convolutional architecture, we have achieved a classification accuracy of 93.56$\%$ on the CIFAR-10 dataset, surpassing the performance of state-of-the-art Gaussian process methods. 
\end{itemize}

Our code which encompasses comprehensive details and the full implementation of our NOVI-DGP approach, can be accessed on our public GitHub repository \href{https://github.com/studying910/NOVI-DGP}{https://github.com/studying910/NOVI-DGP}.

\section{Preliminary} 
\label{sec: background}
In this section, we  present necessary notations and settings on single-layer Gaussian Processes (GPs) and Deep Gaussian Processes (DGPs), then we point out the flaws of current model and introduce our motivation.

\subsection{Gaussian Processes}
\label{subsec: Preliminary on single layer gaussian process}
%\vspace{-2mm}
Let a random function $f: \mathbb{R}^D \rightarrow \mathbb{R}$ map $N$ training inputs $\mathbf{X} \triangleq \{\mathbf{x}_n\}_{n=1}^{N}$ to a collection of noisy observed outputs $\mathbf{y} \triangleq \{ y_n\}_{n=1}^{N}$. In general, a zero mean GP prior is imposed on the function $f$, i.e., $f \sim \mathcal{G P}(0, k)$ where $k$ represents a covariance function $k:\mathbb{R}^D \times \mathbb{R}^D \rightarrow \mathbb{R}$.
Let $\mathbf{f} \triangleq (f(\mathbf{x}_1),...,f(\mathbf{x}_N))^\top$ represent the latent function values at the inputs $\mathbf{X}$. This assumption yields a multivariate Gaussian prior over the function values $p(\mathbf{f}) =\mathcal{N}(\mathbf{f}|\mathbf{0}, \mathbf{K}_{\mathbf{X}\mathbf{X}})$ where 
$\left[ \mathbf{K}_{\mathbf{XX}} \right] _{ij}=k\left( \mathbf{x}_i,\mathbf{x}_j \right)$. In this work, we suppose $\mathbf{y}$ is contaminated by an i.i.d  noise, thus $p(\mathbf{y}|\mathbf{f})=\mathcal{N}(\mathbf{y}|\mathbf{f}, \sigma ^2\mathbf{I})$  where $\sigma^2$ is the noise variance. %Figure \ref{fig:model}($a$) is a graphical illustration  of a typical GP . 
The GP  posterior of the latent output $p\left( \mathbf{f}|\mathbf{y} \right)$ has a closed-form solution \cite{Rasmussen06} but suffers from $\mathcal{O}(N^3)$ computational cost and $\mathcal{O}(N^2)$ storage requirement, thus limiting its scalability to big data.
 %The computational cost of the exact inference of GP is $O(N^3)$, rendering it computationally infeasible for large dataset. 

 %Owing to the reasons above, 
Advanced sparse methods have been developed to set so-called \emph{inducing points} $\mathbf{z} = \{\mathbf{z}_m\}_{m=1}^{M}$ from the input space and the associated inducing outputs known as \emph{inducing variables}: $\mathbf{u} = \{u_m = f(\mathbf{z}_m)\}_{m=1}^{M}$ \cite{Titsias09,Snelson06,candela05}, with a time complexity of $\mathcal{O}(NM^2)$. In this \emph{Sparse GPs} (SGPs) paradigm \cite{Titsias09}, %as shown in Figure \ref{fig:model}($b$), 
\emph{inducing variables} $\mathbf{u}$ share a joint multivariate Gaussian distribution with $\mathbf{f}$: $p(\mathbf{f}, \mathbf{u}) = p(\mathbf{f}|\mathbf{u})p(\mathbf{u})$
where the condition is specified as, 
\begin{equation} 
\label{condi}
p(\mathbf{f}|\mathbf{u}) = \mathcal{N}(\mathbf{f}|\mathbf{K}_{\mathbf{X}\mathbf{Z}}\mathbf{K}_{\mathbf{Z}\mathbf{Z}}^{-1}\mathbf{u}, \mathbf{K}_{\mathbf{X}\mathbf{X}} - \mathbf{K}_{\mathbf{X}\mathbf{Z}}\mathbf{K}_{\mathbf{Z}\mathbf{Z}}^{-1}\mathbf{K}_{\mathbf{Z}\mathbf{X}}),
\end{equation}
and $p\left( \mathbf{u} \right) =\mathcal{N} \left(\mathbf{u}| \mathbf{0},\mathbf{K}_{\mathbf{ZZ}} \right) $ is the prior over the inducing outputs. 
              
%In SGPs the posterior distribution of the latent function $p(\rvf|\vy)$ is intractable since $\mathbf{u}$ need be marginalized, resulting in the following integral:
%\begin{equation}
 % \displaystyle
  %  p(\rvf | \vy)=\int p(\rvf | \rvu) p(\rvu|\vy) d \rvu
%\end{equation} 
%which is difficult to solve since the posterior distribution of inducing variables $p(\mathbf{u} | \mathbf{y})$ is intractable.

To solve  the intractable posterior distribution of inducing variables $p(\mathbf{u} | \mathbf{y})$, \emph{Sparse variational GPs} (SVGPs) \cite{Titsias09, pmlr-v38-hensman15} reformulate the posterior inference problem as variational inference (VI) and confine the variational distribution to be $q(\mathbf{f}, \mathbf{u}) = p(\mathbf{f}| \mathbf{u}) q({\mathbf{u}})$ \cite{Lawrence13,Titsias09,Yarin14,salimbeni2017doubly}. This method approximates $q(\mathbf{u}) = \mathcal{N}(\mathbf{m}, \mathbf{S})$ \cite{pmlr-v38-hensman15,deisenroth2015distributed,Yarin14,Lawrence13,NghiaICML16,HoangICML16,Titsias09a},  then a Gaussian marginal\footnote{The solution is given in App. A} is obtained by maximizing the evidence lower bound (ELBO) \cite{hoffman2013stochastic}. 

\subsection{Deep Gaussian Processes}
\label{subsec: deep gaussian process}
%\vspace{-2mm}

A multi-layer DGP model is a hierarchical composition of GP models constructed by stacking the muti-output SGPs together \cite{damianou2013deep}. %as shown in Figure \ref{fig:model}($c$). 
Consider a model with $L$ layers and $D_\ell$ independent random functions in layer $\ell = 1,\dots,L$ such that output of the $(\ell-1)^\mathrm{th}$ layer $\mathbf{F}_{\ell-1}$ is used as an input to the $\ell^\mathrm{th}$ layer, i.e., $\mathbf{F}_{\ell}\triangleq \{ \mathbf{F}_{\ell ,1}=f_{\ell ,1} \left( \mathbf{F}_{\ell -1}\right),\cdots ,\mathbf{F}_{\ell ,D_{\ell}}=f_{\ell, D_{\ell}}\left( \mathbf{F}_{\ell -1} \right) \}$, where $f_{\ell ,d}\sim \mathcal{G} \mathcal{P} (0,k_{\ell})$ for $d = 1,\dots,D_\ell$ and $\mathbf{F}_0 \triangleq \mathbf{X}$. 
%Compared to single-layer GPs, the output of each layer is a vector $\mathbf{F}_{\ell-1}$, and its components are independent of each other. 
The inducing points and corresponding inducing variables for  each layer are denoted by $\boldsymbol{\mathcal{Z}}\triangleq\{\mathbf{Z}_{\ell}\}_{\ell=1}^L$ and $\mathcal{U} \triangleq\{\mathbf{U}_{\ell}\}_{\ell=1}^L$ respectively where $\mathbf{U}_{\ell}\triangleq \left\{ \mathbf{U}_{\ell ,1}=f_{\ell, 1}\left( \mathbf{Z}_{\ell}\right) ,\cdots, \mathbf{U}_{\ell ,D_{\ell}} = f_{\ell, D_{\ell}} \left(\mathbf{Z}_{\ell} \right) \right\} $. Let $\mathcal{F} \triangleq\{\mathbf{F}_{\ell}\}_{\ell=1}^L$, the DGP model design yields the following joint model density, 
\begin{equation}
\label{dgplikelihood}
     p(\mathbf{y},\mathcal{F} ,\mathcal{U} )=p\left( \mathbf{y}|\mathbf{F}_L \right) \prod_{\ell =1}^L{p}(\mathbf{F}_{\ell}|\mathbf{F}_{\ell -1},\mathbf{U}_{\ell})p\left( \mathcal{U} \right). 
\end{equation}
Here we place independent GP priors within and across layers on $\mathcal{U}$: $p( \mathcal{U} ) =\prod_{l=1}^L{p( \mathbf{U}_l )}=\prod_{l=1}^L{\prod_{d=1}^{D_{\ell}}{\mathcal{N} \left(\mathbf{U}_{\ell,d}| 0,\mathbf{K}_{\mathbf{Z}_{\ell}\mathbf{Z}_{\ell}} \right)}}$ and the condition similar to Equation (\ref{condi}) is defined as follows, 
\begin{equation}
\begin{array}{l}
p\left(\mathbf{F}_{\ell} \mid \mathbf{F}_{\ell-1}, \mathbf{U}_{\ell}\right)=\prod_{d=1}^{D_{\ell}} \mathcal{N}\left(\mathbf{F}_{\ell, d} \mid \mathbf{K}_{\mathbf{F}_{\ell-1}\mathbf{Z}_{\ell}} \mathbf{K}_{\mathbf{Z}_{\ell} \mathbf{Z}_{\ell}}^{-1} \mathbf{U}_{\ell, d}\right. ,\\
\left.\mathbf{K}_{\mathbf{F}_{\ell-1}\mathbf{F}_{\ell-1}} -\mathbf{K}_{\mathbf{F}_{\ell-1}\mathbf{z}_{\ell} } \mathbf{K}_{\mathbf{Z}_{\ell} \mathbf{Z}_{\ell}}^{-1} \mathbf{K}_{\mathbf{Z}_{\ell} \mathbf{F}_{\ell-1}}\right).
\end{array}
\end{equation}
As an extension of Variational Inference with DGPs, DSVI \cite{salimbeni2017doubly} approximates the posterior by requiring the distribution across the inducing outputs to be a-posteriori Gaussian and independent amongst distinct GPs to obtain an analytical ELBO (known as the mean-field assumption \cite{opper2001advanced,hoffman2013stochastic}, %see also in Figure \ref{fig:model}($d$)), 
$q\left( \mathbf{U}_{\ell ,1:D_{\ell}} \right) =\mathcal{N} \left( \mathbf{m}_{\ell, 1:D_{\ell}},\mathbf{S}_{\ell ,1:D_{\ell}} \right) $, where $\mathbf{m}_{\ell ,1:D_{\ell}}$ and $\mathbf{S}_{\ell ,1:D_{\ell}}$ are variational parameters. By iteratively sampling the layer outputs and utilizing the reparameterisation trick \cite{kingma2013auto}, DSVI enables scalability to big datasets. 

The variational posterior distribution $q(\mathcal{U})$ in traditional approximation approaches for Deep Gaussian Process (DGP) models assumes that the distribution follows a mean-field Gaussian, which simplifies the analytical marginalization of the inducing outputs. However, this assumption is overly strict and may limit the effectiveness and expressiveness of the model. By Bayes' Rule, the true posterior distribution can be expressed in a more complex form that is not necessarily Gaussian,
\begin{equation}
\label{posterior}
p\left( \mathcal{U} |\mathbf{y} \right) =\frac{p\left( \mathcal{U} \right) p\left( \mathbf{y}|\mathcal{U} \right)}{p\left( \mathbf{y} \right)}=\frac{\int{p\left( \mathbf{y},\mathcal{F} ,\mathcal{U} \right) d\mathcal{F}}}{p\left( \mathbf{y} \right)},
\end{equation}
where $d\mathcal{F}=d\mathbf{F}_1d\mathbf{F}_2\cdot\cdot\cdot d \mathbf{F}_L$. In Deep Gaussian Process (DGP) models, the likelihood term $p(\mathbf{y}|\mathcal{U})$ in the posterior equation (\ref{posterior}) is difficult to compute because the latent functions $\mathbf{F}_1,\cdots, \mathbf{F}_{L-1}$ are inputs to a non-linear kernel function. Additionally, empirical evidence suggests that the true posterior distribution $p(\mathcal{U}|\mathbf{y})$ is often non-Gaussian, which makes it even more challenging to compute. 

To address this issue, we introduce a novel variational family that balances computational efficiency and improved expressiveness, while also ensuring accurate error control, based on the concept of Operator Variational Inference (OVI) \cite{ranganath2016operator}. Furthermore, our approach includes the learning of preservable transformations and the generation of approximate posterior samples through neural networks, as detailed in Section \ref{sec:sd} and Section \ref{training}.

\section{OVI and Stein Discrepancy}
\label{sec:sd}
\vspace{-2mm}
Before using OVI and Stein Discrepancy to develop a unique inference strategy for DGP model, we provide a quick introduction to these concepts that form the foundation of our method. 

\begin{definition}
\label{def:stein_eqn}
\cite{ranganath2016operator} Let $p(\boldsymbol{x})$ be a probability density supported on $\mathcal{X} \subseteq \mathbb{R}^d$ and $\boldsymbol{\phi }: \mathcal{X} \rightarrow \mathbb{R}^d$ be  a  differentiable function, we define Langevin-Stein Operator as:
\begin{equation}
\displaystyle
\mathcal{A}_p \boldsymbol{\phi }( \boldsymbol{x} ) \triangleq (\nabla_{\boldsymbol{x}} \log p( \boldsymbol{x} )) ^T\boldsymbol{\phi } ( \boldsymbol{x}) +\mathrm{Tr}( \nabla _{\boldsymbol{x}} \boldsymbol{\phi }( \boldsymbol{x})).
\end{equation}
\end{definition}
%\begin{comment}
%\begin{comment}
\begin{lemma}
\label{lemma1}
\cite{liu2016stein} Let $p(\boldsymbol{x})$ be a probability density function supported on $\mathcal{X} \subseteq \mathbb{R}^d$, and $\boldsymbol{\phi}: \mathcal{X} \rightarrow \mathbb{R}^d$ be a differentiable function. Suppose that $\int_{\partial \mathcal{X}} p(\boldsymbol{x}) \boldsymbol{\phi}(\boldsymbol{x}) d\boldsymbol{x} = \boldsymbol{0}$, where $\partial \mathcal{X}$ represents the boundary of $\mathcal{X}$. Under these conditions, Stein's identity can be expressed as
\begin{align}
\label{equ:steq1}
\mathbb{E} _{\boldsymbol{x}\sim p}\left[ \mathcal{A} _p\boldsymbol{\phi }( \boldsymbol{x} ) \right] =0.
\end{align}
\end{lemma}

When considering the expectation of $\mathcal{A} _p\boldsymbol{\phi}(\boldsymbol{x})$ under $\boldsymbol{x}\sim q$, where $q(\boldsymbol{x})$ is another probability density supported on $\mathcal{X} \subseteq \mathbb{R}^d$, the implication of Lemma \ref{lemma1} is that for arbitrary $\boldsymbol{\phi}$, the expectation will not be necessarily equal to zero. Instead, the magnitude of $\mathcal{A} _p\boldsymbol{\phi}(\boldsymbol{x})$ under $\boldsymbol{x}\sim q$ reflects the difference between probability distributions $p$ and $q$. Thus, we can define a discrepancy measure, referred to as Stein discrepancy, to capture the difference between the target distribution and its approximation.
%\end{comment}

%\begin{comment}

\begin{definition}
(Stein's Discrepancy) \cite{liu2016stein} Let $p(\boldsymbol{x})$, $q(\boldsymbol{x})$ be probability densities supported on $\mathcal{X} \subseteq \mathbb{R}^d$. Stein discrepancy is defined as the maximum violation of Stein's identity in a proper function set $\mathcal{G}$ for any differentiable function $\boldsymbol{\phi}:\mathcal{X} \rightarrow \mathbb{R}^d$, i.e.,
\begin{equation}
\mathcal{S} \left( q,p \right) \triangleq \underset{\boldsymbol{\phi }\in \mathcal{G}}{\mathrm{sup}}\; \mathbb{E} _{\boldsymbol{x}\sim q}[ \mathcal{A} _p\boldsymbol{\phi }(\boldsymbol{x})].
\label{eq:sd}
\end{equation}
\end{definition}

The selection of the function set $\mathcal{G}$ is crucial here, as it determines the discriminative power and computational feasibility of the Stein discrepancy. Traditionally, $\mathcal{G}$ consists of functions with bounded Lipschitz norms. However, this approach poses a challenging functional optimization problem that is computationally intractable or demands special considerations. Similar to prior approaches \cite{grathwohl2020learning}, we adopt the $\mathcal{L}_2$ space as the function space $\mathcal{F}$ in the Stein discrepancy (\ref{eq:sd}) and represent $\boldsymbol{\phi}$ with a neural network $\boldsymbol{\phi}_{\boldsymbol{\eta}}$ as a discriminator to maximize
\begin{equation}
\label{eq:lsd}
\begin{split}
&\mathrm{LSD}\left( q,p;\boldsymbol{\eta} \right) \triangleq \mathbb{E} _{\boldsymbol{x}\sim q}[ (\nabla _{\boldsymbol{x}}\log p\left( \boldsymbol{x} \right)) ^T\boldsymbol{\phi }_{\boldsymbol{\eta}}\left( \boldsymbol{x} \right) \\&+\mathrm{Tr}\left( \nabla _{\boldsymbol{x}}\boldsymbol{\phi }_{\boldsymbol{\eta}}\left( \boldsymbol{x} \right) \right)],
\end{split}
\end{equation}
with respect to the parameters ${\boldsymbol{\eta}}$ of the neural network. This approach, known as  the Learned Stein Discrepancy (LSD) \cite{grathwohl2020learning}, uses neural networks as discriminators to parameterize $\boldsymbol{\phi}$ in the Stein discrepancy equation (\ref{eq:sd}). However, neural networks are not inherently square integrable and do not vanish at infinity. In order to satisfy the conditions of Stein's identity \cite{liu2016stein}, an $\mathcal{L}_2$ regularizer is applied to the LSD with a regularization strength $\lambda \in \mathbb{R}^+$, resulting in a Regularized Stein Discrepancy (RSD)
\begin{equation}
\label{rsd}
\begin{split}
&\mathrm{RSD}\left( q,p;\boldsymbol{\eta} \right)\triangleq \mathbb{E} _{\boldsymbol{x}\sim q}[ (\nabla _{\boldsymbol{x}}\log p\left( \boldsymbol{x} \right)) ^T\boldsymbol{\phi }_{\boldsymbol{\eta}}\left( \boldsymbol{x} \right)\\ &+\mathrm{Tr}\left( \nabla _{\boldsymbol{x}}\boldsymbol{\phi }_{\boldsymbol{\eta}}\left( \boldsymbol{x} \right) \right)]-\lambda \mathbb{E} _{\boldsymbol{x}\sim q}[ \boldsymbol{\phi }_{\boldsymbol{\eta}}( \boldsymbol{x} ) ^{T}\boldsymbol{\phi }_{\boldsymbol{\eta}}\left( \boldsymbol{x} \right)] .
\end{split}
\end{equation}

In Bayesian posterior inference, we aim to approximate the true posterior $p$ using an approximate posterior $q_{\boldsymbol{\theta}}$, parameterized by variational parameters ${\boldsymbol{\theta}} \in \varTheta$, where $\varTheta$ is a set of possible parameterizations. Stein divergence, as defined in equation (\ref{eq:sd}), is often used as the objective function for the Operator Variational Inference (OVI) algorithm \cite{ranganath2016operator}. OVI is a black-box algorithm that leverages operators to optimize any operator-based objective, with the benefits of data subsampling and the capability to operate with a wider class of posterior approximations that do not require tractable densities. By combining the parameterizations of the set $\varTheta$ and the discriminator $\boldsymbol{\phi}_{\boldsymbol{\eta }}$, OVI  solves a minimax problem to find the optimal variational parameters $\boldsymbol{\theta}^\star$, i.e.,
\begin{equation}
\boldsymbol{\theta} ^\star=
\mathrm{arg}
\mathop {\mathrm{inf}} \limits_{\boldsymbol{\theta} \in \varTheta}(\mathop {\mathrm{sup}} \limits_{\boldsymbol{\eta } }\,\, \mathbb{E} _{\boldsymbol{x}\sim q_{\boldsymbol{\theta}}}\left[ \mathcal{A} _p\boldsymbol{\phi}_{\boldsymbol{\eta} }\left( \boldsymbol{x} \right) \right]). 
\end{equation}

\section{Deep Gaussian Processes with Neural Operator Variational Inference}
\label{training}
In this section, we present the algorithmic design for performing Bayesian inference on the posterior $p\left( \mathcal{U}|\mathcal{D} \right)$ of DGPs. We adopt the notation introduced in Section \ref{subsec: deep gaussian process}, where $\mathcal{D} =\left\{ \boldsymbol{x}_n,y_n \right\} _{n=1}^{N} $  denotes the training dataset, $ \mathcal{U} \triangleq \left\{ \mathbf{U}_{\ell} \right\} _{\ell =1}^{L}
$ represents the inducing variables, and $\boldsymbol{\nu}$ denotes the hyperparameters of the DGP model, including the inducing point locations, kernel hyperparameters, and noise variance.
\subsection{Neural network as Generators}
Consider a reference distribution $q_0(\boldsymbol{\epsilon})$ that generates noise $\boldsymbol{\epsilon} \in \mathbb{R}^{d_0}$. We represent the neural network that generates the posterior distribution, along with its parameters, as $g_{\boldsymbol{\theta}}$. If a noise vector is passed through this network, the resulting distribution of generated samples can be expressed as $\mathcal{U} = g_{\boldsymbol{\theta}}(\boldsymbol{\epsilon})$.
In summary, our setup is shown as follows,
\begin{equation}
\boldsymbol{\epsilon }\sim q_0\left( \boldsymbol{\epsilon } \right) ,\qquad     g_{\boldsymbol{\theta}}\left( \boldsymbol{\epsilon } \right) =\mathcal{U} \sim q_{\boldsymbol{\theta}}\left( \mathcal{U} \right). 
\end{equation}
Compared to traditional machine learning methods, such as grid-based or orthogonal designs, neural networks are recognized for their superior capability to model distributions, enabling them to learn implicit posterior distributions from data. As generators, neural networks can transform simple distributions such as Gaussian or uniform distributions, making them highly versatile and widely used in deep generative models \cite{huszar2017variational,mescheder2017adversarial,titsias2019unbiased,cybenko1989approximation,lu2020universal,perekrestenko2020constructive,yang2022capacity,yang2020potential,xie2021semisupervised}.As mentioned earlier, the high-dimensional and non-mean-field nature of the posterior distribution in deep generative models makes KL divergence unsuitable as a measure for the fit between the generative distribution $q_{\boldsymbol{\theta}}\left( \mathcal{U} \right)$ and the true posterior $p(\mathcal{U} |\mathcal{D})$.Therefore, using OVI and RSD to construct a better objective is a reasonable approach.

\subsection{Training Schedule} 
In Section \ref{sec:sd}, we provided a review of OVI, a method that uses Langevin-Stein operator to enable  a more flexible representation of the posterior geometry beyond the commonly used Gaussian distribution in vanilla VI. We extend this technique to the context of inducing points posterior inference for DGP models by iteratively training the discriminator and generator parameters to optimize the fit of the posterior to the data. Using the definition of RSD, we can construct an objective whose expectation value is zero if and only if the true posterior $p(\mathcal{U} |\mathcal{D})$ and the approximate distribution $q(\mathcal{U})$ are equivalent. During training, our goal is to minimize this objective
\begin{equation}
\mathcal{L} (\boldsymbol{\theta },\boldsymbol{\nu })=\underset{\boldsymbol{\eta }}{\mathrm{max}}(\text{RSD}\,(q_{\boldsymbol{\theta }}(\mathcal{U} ),p(\mathcal{U} |\mathcal{D} ,\boldsymbol{\nu });\boldsymbol{\phi }_{\boldsymbol{\eta }})),
\label{eq:l}
\end{equation}
with respective to $\boldsymbol{\theta},\boldsymbol{\nu}$. Based on equation (\ref{rsd}), we observe that equation (\ref{eq:l}) is a minmax problem. To solve it, we need to find the supremum on the right-hand side of the equation while jointly optimizing the inducing points posterior distribution and 
other model parameters $\boldsymbol{\nu}$ such as the Gaussian process kernel parameters. To achieve this, we separate the optimization of the discriminator $\boldsymbol{{\phi }_{\eta}}$ and generator $g_{\boldsymbol{\theta}}$ to enable optimal estimation of these parameters. Since the other model 
parameters $\boldsymbol{\nu}$ are point estimates, we utilize Monte Carlo sampling and maximum likelihood estimation to optimize them after optimizing the discriminator and generator. We present the main idea of our algorithm and its pseudocode in Algorithm \ref{alg:DGP-NS} and refer to as \emph{Neural Operator Variational Inference} (NOVI) for DGP models.

In our implementation, we utilize Monte Carlo method to estimate the objective (\ref{eq:l}) and RSD (\ref{rsd}), 
\begin{small}
\begin{equation}
\label{Hutchinson}
\begin{array}{l}
\widehat{\operatorname{RSD}}\left( q_{\boldsymbol{\theta}},p ; \boldsymbol{\phi}_{\boldsymbol{\eta}}\right)=\frac{1}{K} \sum_{k=1}^{K}\left(\left.\nabla_{\mathcal{U}} \log p(\mathcal{U} \mid \mathcal{D}, \boldsymbol{\nu})^{T}\right|_{\mathcal{U}=\mathcal{U}^{k}} \boldsymbol{\phi}_{\boldsymbol{\eta}}\left(\mathcal{U}^{k}\right)\right. \\
\left.+\mathbb{E}_{\omega \sim \mathcal{N}(0, \mathrm{I})}\left(\left.\omega^{T} \nabla_{\mathcal{U}} \phi_{\boldsymbol{\eta}}(\mathcal{U})\right|_{\mathcal{U}=\mathcal{U}^{k}} \omega\right)\right)\\-\lambda \frac{1}{K} \sum_{k=1}^{K}\left(\boldsymbol{\phi}_{\boldsymbol{\eta}}\left(\mathcal{U}^{k}\right)^{T} \boldsymbol{\phi}_{\boldsymbol{\eta}}\left(\mathcal{U}^{k}\right)\right), 
\end{array}
\end{equation}
\end{small}
where $\boldsymbol\phi_{\boldsymbol\eta^\star}$ is the supremum of RSD estimate and the gradient with $\boldsymbol\theta$ and $\boldsymbol\nu$ is computed via automatic differentiation. To compute the expensive divergence of $\boldsymbol{\phi}_{\boldsymbol{\eta}}$ in Equation (\ref{Hutchinson}), we use the Hutchinson estimator \cite{hutchinson1989stochastic}, which provides a stochastic estimate of the trace of a matrix and reduces the time complexity from $\mathcal{O}(D^2)$ to $\mathcal{O}(D)$, where $D$ is the dimensionality of the matrix. In \textbf{Theorem} \ref{thm:1}, we prove that the score function $\nabla_{\mathcal{U}}\log p(\mathcal{U}|\mathcal{D},\boldsymbol{\nu})$ can be evaluated by Monte Carlo method, which demonstrates that RSD is a suitable objective for updating the parameters of the generator network. 
\begin{algorithm}[ht!]
\caption{NOVI for DGP models}
\label{alg:DGP-NS}
\begin{algorithmic}
   \STATE {\bfseries Input:} training data $\mathcal{D} =\left\{ \boldsymbol{x}_n,y_n \right\} _{n=1}^{N} $, penalty parameter $\lambda$, $n_c$ number of iterations for training the discriminator, learning rate $
\alpha,\beta, \gamma 
$, M batch size, sample number K
   \STATE {\bfseries Initialize} discriminator $\boldsymbol{\eta}$, generator $\boldsymbol{\theta}$, DGP hyperparameters $\boldsymbol{\nu}$ 
   \REPEAT
   \FOR{$j=1$ {\bfseries to} $n_c$}
   \STATE Sample  a minibatch $
\left\{ \boldsymbol{x}_i,y_i \right\} _{i=1}^{M}\sim \mathcal{D}$ 
   \STATE Generate i.i.d. standard normal distribution noise $
\boldsymbol{\epsilon} _1 \dots \boldsymbol{\epsilon} _K$ from $q_0
$
   
   \STATE Generate sample  $
g_{\boldsymbol{\theta}}\left( \boldsymbol{\epsilon} _1 \right) \dots g_{\boldsymbol{\theta}}\left( \boldsymbol{\epsilon} _K \right) 
$ from the generator
   \STATE Compute empirical loss $\widehat{\mathrm{RSD}}( q_{\boldsymbol{\theta}},p;\boldsymbol{\phi }_{\boldsymbol{\eta}} ) $
  
   \STATE $\boldsymbol{\eta} \gets \boldsymbol{\eta}-\alpha 
\nabla _{\eta}\widehat{\mathrm{RSD}}( q_{\boldsymbol{\theta}},p;\boldsymbol{{\phi }_{\eta}}) 
$

   \ENDFOR
    \STATE Compute empirical loss $
\widehat{\operatorname{RSD}}\left( q_{\boldsymbol{\theta}},p ; \phi_{\boldsymbol{\eta}^{*}}\right) $
    \STATE $
\boldsymbol{\theta} \gets \boldsymbol{\theta} -\beta \,\,\nabla _{\boldsymbol{\theta}}\widehat{\operatorname{RSD}}\left( q_{\boldsymbol{\theta}},p ; \phi_{\boldsymbol{\eta}^{*}}\right)
$

    \STATE $ \boldsymbol{\nu} \gets \boldsymbol{\nu} -\gamma \,\,\frac{1}{K}\sum\nolimits_{k=1}^K{\nabla _{\boldsymbol{\nu}}\log p(\mathbf{y},\mathcal{U} ^k|\boldsymbol{\nu} )}
$
    \UNTIL{$\boldsymbol{\theta},\boldsymbol{\nu}$  converge}
    \end{algorithmic}
\end{algorithm}

\begin{theorem}
\label{thm:1}
Assuming that $\mathcal{U} \in \varOmega$, $\boldsymbol{\nu} \in \varUpsilon$ where $\varOmega$ and $\varUpsilon$ are both compact spaces. We can obtain an asymptotically unbiased estimator for the score function $\nabla_{\mathcal{U}}\log p(\mathcal{U} |\mathcal{D},\boldsymbol{\nu})$ in equation \ref{Hutchinson}, which converges in probability to the true value. (detailed proof can be seen in App. B): %\ref{theorem-1}
\begin{equation}
\label{score func}
\begin{split}
&\nabla _{\mathcal{U}}\log p(\mathcal{U} |\mathcal{D}, \boldsymbol{\nu}) \approx -\left( \mathbf{\Delta}_1,\dots,\mathbf{\Delta }_{\ell},\dots,\mathbf{\Delta }_L \right)^\top \\&+\mathbf{\nabla}_{\mathcal{U}}\log \sum\nolimits_{s=1}^S{p(\mathbf{y}|\widehat{\mathbf{F}}_{L}^{\left( s \right)})},
\end{split}
\end{equation}
where $\boldsymbol{\Delta}_{\ell}=$\\
\begin{small}$((\mathbf{K}_{\mathbf{Z}_{\ell} \mathbf{Z}_{\ell}}{ }^{-1} \mathbf{U}_{\ell, 1})^\top,...,(\mathbf{K}_{\mathbf{Z}_{\ell} \mathbf{Z}_{\ell}}{ }^{-1} \mathbf{U}_{\ell, d})^\top,...,(\mathbf{K}_{\mathbf{Z}_{\ell} \mathbf{Z}_{\ell}}{ }^{-1} \mathbf{U}_{\ell, D_{\ell}})^\top)$ \end{small} and \begin{small} $\widehat{\mathbf{F}}_{\ell,d}^{( s )} \sim
\mathcal{N} (\mathbf{K}_{\widehat{\mathbf{F}}_{\ell -1}\mathbf{Z}_{\ell}}\mathbf{K}_{\mathbf{Z}_{\ell}\mathbf{Z}_{\ell}}^{-1}\mathbf{U}_{\ell,d},\mathbf{K}_{\widehat{\mathbf{F}}_{\ell-1}\widehat{\mathbf{F}}_{\ell -1}}-\mathbf{K}_{\widehat{\mathbf{F}}_{\ell -1}\mathbf{Z}_{\ell}}\mathbf{K}_{\mathbf{Z}_{\ell}\mathbf{Z}_{\ell}}^{-1}\mathbf{K}_{\mathbf{Z}_{\ell}\widehat{\mathbf{F}}_{\ell -1}}) $ \end{small}  for $\ell = 1 ,\dots, L$, $S$ is the number of samples involved in estimation.
\end{theorem} 

In Monte Carlo estimation of the log likelihood function, bias can arise due to the logarithmic transformation on the likelihood function, which is not explicitly defined in the DGP model. This bias can affect other DGP inference objectives, such as DSVI \cite{salimbeni2017doubly} .

However, in our method based on the score function, the gradient operator cancels out the bias introduced by the logarithmic transformation. The fact that our model is asymptotically unbiased, as proven by theorem \ref{thm:1}, is beneficial for mini-batch methods that rely on random sub-sampling of data. This property enhances both the scalability and accuracy of our method.

To ensure the convergence of this estimate, we propose to introduce a constraint function $c(\cdot)$ to restrict the parameter space of the objective function, which, to our knowledge, has not been considered in previous work. This constraint function $c(\cdot)$ is designed to confine $\mathcal{U}$ and $\boldsymbol{\nu}$ within a compact space. For instance, if we adopt a squared exponential kernel function $k_{SE}(x, x') = \sigma_f^2 \exp\left(-\frac{(x - x')^2}{2l^2}\right)
 $with length scale $l$, we can apply a clip function as a constraint for an appropriate closed interval $[P,Q]$, namely,

\begin{equation}
     clip\left( l \right) =\begin{cases}
	P\,\,  \text{if}\,\,  l<P\\
	l\,\,  \text{if}\,\,  P\leqslant l\leqslant Q\\
	Q\,\,  \text{if}\,\,  l>Q\\
    \end{cases}.
\end{equation}

As for the generated $\mathcal{U}$ by the neural network generator, we can also apply such a constraint in the last layer to ensure the convergence of the score function estimate.

\subsection{Prediction} 
To obtain the final layer density for predicting the value of the test data $\mathcal{D}^\star = \left\{\mathbf{x}_n^\star, y_n^\star \right\}_{n=1}^{T}$, we first sample from the optimized generator and transform the input locations $\mathbf{x}_n$ to the test locations $\mathbf{x}_n^\star$ using equation (\ref{dgplikelihood}). We subsequently compute the function values at the test locations, which are represented as $\mathbf{F}_{\ell}^\star$. Finally, we use  equation (\ref{QF}) to estimate the density of the final layer, which enables us to make predictions for the test data
\begin{equation}
\label{QF}
\begin{split}
&q(\mathbf{F}_{L}^{\star})=\\&\int{\prod\nolimits_{\ell =1}^L{\prod\nolimits_{d=1}^{D_{\ell}}{p(\mathbf{F}_{\ell ,d}^{\star}|\mathbf{F}_{\ell -1}^{\star},\mathbf{U}_{\ell ,d}) q_{\theta ^{\star}}\left( \mathbf{U}_{\ell ,d} \right) d\mathbf{F}_{\ell -1}^{\star}d\mathbf{U}_{\ell ,d}}}}
\end{split}
\end{equation}
where $\boldsymbol{\theta}^\star$ is the optimal of the generator and the first term of the integral $p( \mathbf{F}_{\ell,d}^\star|\mathbf{F}_{\ell -1}^\star,\mathbf{U}_{\ell ,d} )$ is conditional Gaussian. We leverage this consequence to draw samples from $q\left(\mathbf{F}_{L}^\star\right)$, and further perform the sampling using re-parameterization trick \cite{salimbeni2017doubly,rezende2014stochastic,kingma2015variational}. Specifically, we first sample $\boldsymbol\epsilon ^{\ell}\sim \mathcal{N} (0,\mathbf{I}_{D^{\ell}})$ and $\mathcal{U} \sim q_{\boldsymbol\theta ^\star}(\mathcal{U})$, then recursively draw the sampled variables $\widehat{\mathbf{F}}_{\ell,d}^\star\sim p(\mathbf{F}_{\ell,d}^\star|\widehat{\mathbf{F}}_{\ell-1}^\star,\mathbf{U}_{\ell,d} ) $ for $\ell = 1 ,\dots, L$ as,
\begin{equation}
\begin{split}
&\widehat{\mathbf{F}}_{\ell,d}^{\star}=\mathbf{K}_{\widehat{\mathbf{F}}_{\ell-1}^{\star}\mathbf{Z}_{\ell}}\mathbf{K}_{\mathbf{Z}_{\ell}\mathbf{Z}_{\ell}}^{-1}\mathbf{U}_{\ell,d}
\\&+\boldsymbol{\epsilon} _{\ell}\odot \sqrt{\mathrm{diag}\,(\mathbf{K}_{\widehat{\mathbf{F}}_{\ell -1}^{\star}\widehat{\mathbf{F}}_{\ell -1}^{\star}}-\mathbf{K}_{\widehat{\mathbf{F}}_{\ell -1}^{\star}\mathbf{Z}_{\ell}}\mathbf{K}_{\mathbf{Z}_{\ell}\mathbf{Z}_{\ell}}^{-1}\mathbf{K}_{\mathbf{Z}_{\ell}\widehat{\mathbf{F}}_{\ell -1}^{\star}})},
\label{pre}
\end{split}
\end{equation}
where `$\odot$' represents the Hadamard product, and the square root is element-wise, and  we define $\mathbf{F}_0^\star \triangleq \mathbf{X}^\star$ for the first layer and use $\mathrm{diag}\left(\cdot\right)$ to denote the vector of diagonal elements of a matrix. The diagonal approximation in equation (\ref{pre}) holds since in DGP model, the $i^\mathrm{th}$ marginal of approximate posterior $q( \mathbf{F}_{\left( \ell ,d \right)}^i)$ depends only on the corresponding inputs $\boldsymbol{x}_i$ \cite{candela05}.  
\section{Convergence Guarantees and Bias Control}
\label{section5}

In this section, we provide convergence guarantees and error control for NOVI, detailed proof can be seen in App. C.
%\ref{theorem-23}.

\begin{definition}
The Fisher divergence \cite{sriperumbudur2017density} between two suitably smooth density functions is defined as 
\begin{equation}
FD\left(q, p \right) =\int_{\mathbb{R} ^d}{\left\| \mathbf{\nabla }\log q\left( \boldsymbol{x}  \right) -\mathbf{\nabla }\log p\left( \boldsymbol{x} \right) \right\| _{2}^{2}q\left( \boldsymbol{x}  \right) d\boldsymbol{x} }.
\end{equation} 
\end{definition}

\begin{theorem}
\label{thm:2}
Training the generator with the optimal discriminator corresponds to minimizing the Fisher divergence between $q_{\boldsymbol{\theta}}$ and $p$, and the corresponding optimal loss for equation (\ref{eq:l}) is
\begin{equation}
    \mathcal{L} \left( \boldsymbol{\theta} ,\boldsymbol{\nu} \right) =\frac{1}{4\lambda}FD\left( q_{\boldsymbol{\theta}}\left( \mathcal{U} \right) ,p\left( \mathcal{U} |\mathcal{D} ,\boldsymbol{\nu} \right) \right), 
\end{equation}
where $\lambda \in \mathbb{R}^+$ is a regularization strength defined in equation (\ref{rsd}). 
\end{theorem}

\begin{theorem}
\label{thm:3}
Assuming that $\mathcal{U} \in \varOmega$, $\boldsymbol{\nu} \in \varUpsilon$ where $\varOmega$ and $\varUpsilon$ are both compact spaces. The bias of the estimation for prediction $\widehat{\mathbf{F}}_{L}^\star$ in equation (\ref{pre}) from the DGPs exact evaluation can be bounded by the square root of the Fisher divergence between $q_{\boldsymbol{\theta}}( \mathcal{U})$ and $p\left( \mathcal{U} |\mathcal{D}, \boldsymbol{\nu} \right)$  up to multiplying a constant.
\end{theorem}

\textbf{Theorem} \ref{thm:2} demonstrates that our algorithm is equivalent to minimizing Fisher divergence, while \textbf{Theorem}  \ref{thm:3} guarantees a bounded bias for prediction estimation. Fisher divergence has proven to be a valuable tool in various statistics and machine learning applications, as demonstrated by its use in practical contexts such as generative models \cite{ho2020denoising}, Bayesian inference \cite{sriperumbudur2017density}, and others \cite{huggins2018practical}.

Moreover, Fisher divergence has strong connections to other distance metrics such as total variation \cite{chambolle2004algorithm}, Hellinger distance \cite{beran1977minimum}, and Wasserstein distance \cite{vallender1974calculation}. In fact, it is often a stronger distance metric than these alternatives \cite{ley2013stein}. This means that when the Fisher divergence is smaller, it implies that the other distance metrics, which are weaker than Fisher divergence, will also be smaller. Consequently, using Fisher divergence as a measure provides better error control in comparison to these alternative metrics. The stronger connection of Fisher divergence to other distance metrics allows the proposed approach to capture more subtle differences between probability distributions, leading to more accurate moment estimates and improved performance, especially in high-dimensional distribution estimation \cite{huggins2018practical}. By leveraging the advantages of Fisher divergence, our method achieves enhanced theoretical guarantees and translates them into practical gains in real-world tasks.

\begin{figure*}[ht]
\vskip 0.2in
\begin{center}
\centerline{\includegraphics[width=2\columnwidth]{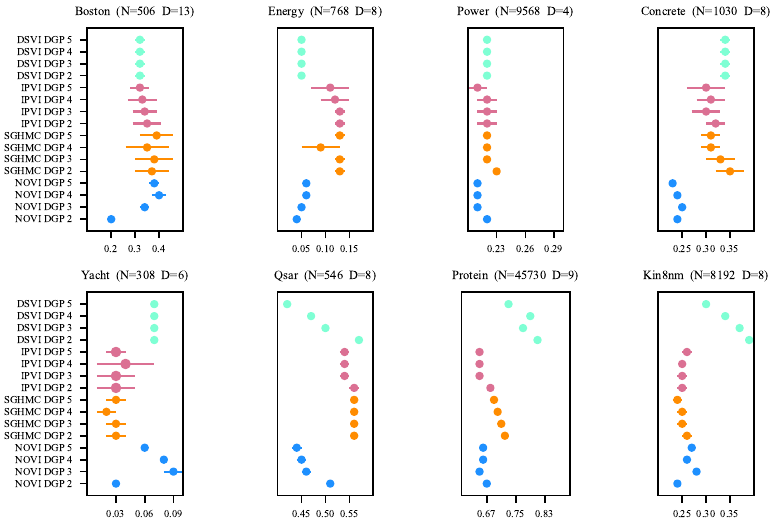}}
\caption{Regression mean test RMSE results by our NOVI method (blue), SGHMC (orange), IPVI(pink) and DSVI (cyan) for DGPs on UCI benchmark datasets. The numbers 2, 3, 4, and 5 represent the layers of DGP methods. Lower is better. The mean is shown with error bars of one standard error.}
\label{fig:regression}
\end{center}
\vskip -0.2in
\end{figure*}

\section{Related Works}

Our method for inference is inspired by two previous works, namely OVI \cite{ranganath2016operator} and Stein Discrepancy\cite{liu2016stein}. However, our approach is distinct in that it specifically focuses on the DGP posterior and develops tailored algorithms to address it. One key challenge in calculating the score function for DGPs is that the likelihood function is not explicit, and thus we propose a stochastic gradient and Monte Carlo sampling method to address this issue (see \textbf{Theorem} \ref{thm:1}). While OVI \cite{ranganath2016operator} introduces a similar objective for inference to RSD, it utilizes a different class of discriminator, and does not employ many of the state-of-the-art techniques we use for scalability, such as the Hutchinson estimator \cite{hutchinson1989stochastic}.

The computation of the term $\mathrm{Tr}\left(\mathbf{\nabla}_{\boldsymbol{x}}\boldsymbol{\phi }_{\boldsymbol{\eta}}\left( \boldsymbol{x} \right) \right)$ in Equation (\ref{eq:lsd}) is computationally expensive as it requires $\mathcal{O} \left( d \right)$ vector-Jacobian products, since each entry of the diagonal of the Jacobian requires computing a separate derivative of $\boldsymbol{\phi }_{\boldsymbol{\eta}}$. To address this issue, we can use the Hutchinson estimator, which only requires one vector-Jacobian product to compute. This estimator can be obtained by multiplying the matrix $\mathbf{\nabla}_{\boldsymbol{x}}\boldsymbol{\phi }_{\eta}\left( \boldsymbol{x} \right)$ by a noise vector twice, as shown in the following identity \cite{hutchinson1989stochastic},
\begin{equation}
\mathrm{Tr}\left( \mathbf{\nabla }_{\boldsymbol{x}}\boldsymbol{\phi }_{\boldsymbol{\eta}}\left( \boldsymbol{x} \right) \right) =\mathbb{E} _{\boldsymbol{\epsilon} \sim \mathcal{N} \left( 0,I \right)}\left[ \boldsymbol{\epsilon} ^T\mathbf{\nabla }_{\boldsymbol{x}}\boldsymbol{\phi }_{\boldsymbol{\eta}}\left( \boldsymbol{x} \right) \boldsymbol{\epsilon} \right].
\end{equation}
This single-sample Monte-Carlo estimator has been widely used in recent years in the machine learning community due to its efficiency and unbiasedness \cite{grathwohl2018ffjord,tsitsulin2019shape,han2017approximating}. The basic principle of this estimator is to approximate the trace by introducing random variables, such as Gaussian distributed vectors. Specifically, we sample the matrix using multiple random vectors and estimate the trace of the matrix-vector product by summing the products of the sampled results. The benefit of this estimation is that it only requires computing the products of the matrix-vector multiplications, rather than explicitly computing the trace of the entire matrix, thereby reducing computational complexity.

\section{Experiments}
In order to evaluate the performance of our proposed method, we conducted empirical evaluations on real-world datasets for both regression and classification tasks, with both small and large datasets. Our method was compared against several other models, including Doubly Stochastic VI (DSVI) \cite{salimbeni2017doubly}, which was used as our baseline model, Implicit Posterior VI (IPVI) \cite{yu2019implicit}, which constructs a variational lower bound using density ratio estimates, and the state-of-the-art SGHMC model \cite{havasi2018inference}. All experiments were conducted with the same hyper-parameters and initializations, and we provide detailed training information in Appendix E.

\subsection{UCI Regression Benchmark}
\label{7A}
In our experiments, we evaluated the performance of the NOVI model on eight UCI regression datasets, which varied in size from 308 to 45,730 data points. We used the average RMSE of the test data as the performance metric, and the results are presented in Figure \ref{fig:regression}. The tabular version of the results can be found in Appendix D.C.

As shown in the figure \ref{fig:regression}, our NOVI method consistently achieves competitive results compared to baselines on the majority of datasets. This is attributed to the key advantages of our approach, which overcomes limitations present in previous methods such as the restrictive nature of mean-field posterior distributions and provides stronger guarantees in terms of convergence and error control using Fisher divergence, as discussed in Section \ref{section5}. 

In the analysis of the 'Boston', 'Energy', 'Yacht', and 'Kin8nm' datasets, our method demonstrates impressive performance with two-layer DGPs. As the number of layers increases, there is some uncertainty observed. However, when compared to other baseline methods, NOVI remains competitive, as illustrated in Table \ref{tab:regression-tabular}. Notably, for larger datasets like 'Power' and 'Protein', deeper NOVI models exhibit superior performance compared to other methods. These results suggest that the effectiveness of our method may vary depending on the dataset characteristics and the optimal model hyperparameters. It is essential to note that performance fluctuations with increasing layers are not exclusive to our model but are also evident in other models such as SGHMC and IPVI.

We have also included additional results for real-world regression datasets in Appendix D.E, further demonstrating the effectiveness of the NOVI model.

\begin{table*}
\centering
\small
% \resizebox{\linewidth}{!}{ %< auto-adjusts font size to fill line
\centering
\begin{tabular}{llcccccc}   % @{}lccccc@{}
\toprule
Data Set & Model & Time (L=3) & Iter(L=3) & Acc (L=3) & Time (L=4) & Iter (L=4) & Acc (L=4) \\
\midrule
 & DSVI & 0.34s/iter & 20K & 97.17 & 0.54s/iter & 20K & 97.41 \\
MNIST & IPVI & 0.49s/iter & 20K & 97.58 & 0.62s/iter & 20K & 97.80 \\
 & SGHMC & 1.14s/iter & 20K & 97.25 &  1.22s/iter & 20K & 97.55 \\
 & NOVI (ours) & 0.38s/iter & 10K & \textbf{98.04} & 0.50s/iter & 10K & \textbf{98.21} \\
\hline

 & DSVI & 0.34s/iter & 20K & 87.45 & 0.50s/iter & 20K & 87.99 \\
Fashion-MNIST & IPVI & 0.48s/iter & 20K & 88.23 & 0.61s/iter & 20K & 88.90 \\
 & SGHMC & 1.21s/iter & 20K & 86.88 & 1.25s/iter & 20K & 87.08 \\
 & NOVI (ours) & 0.40s/iter & 10K & \textbf{89.36} & 0.55s/iter & 10K & \textbf{89.65} \\

\hline

 & DSVI & 0.43s/iter & 20K & 51.33 & 0.66s/iter & 20K & 51.79 \\
CIFAR-10 & IPVI & 0.62s/iter & 20K & 52.79 & 0.78s/iter & 20K & 53.27 \\
 & SGHMC & 8.04s/iter & 20K & 52.34 & 8.61s/iter & 20K & 52.81 \\
 & NOVI (ours) & 0.43s/iter & 10K & \textbf{53.42} & 0.52s/iter & 10K & \textbf{53.62} \\

\bottomrule
\end{tabular}
% \resizebox
\caption{Mean test accuracy ($\%$) and training details achieved by DSVI, SGHMC and NOVI (ours) DGP model for three image classification datasets. Batch size is set to $256$ for all methods. L denotes the number of hidden layers. Our proposed method can also be combined with convolution kernels \cite{kumar2018deep} to obtain a better result, for a fair comparison, we have not implemented here but in next part.}
\label{tab:image-classification}
\end{table*}
\begin{figure*}[t]
    \centering
    \includegraphics[width=0.99\linewidth]{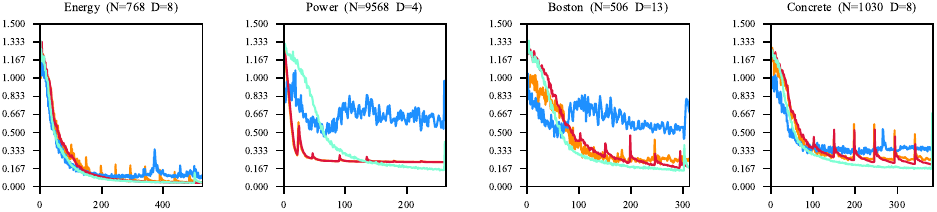}
    \caption{The mean RMSE comparison of NOVI (test: orange, train: red) with Monte Carlo log-likelihood maximization method (test: blue, train: cyan) using $2$-layer DGP model on four UCI regression datasets.}
    \label{fig:ablation-regression}
\end{figure*}

\subsection{Large-Scale Regression}
 Using mini-batch algorithm, our method can also be extended to datasets at the million level. Our evaluation of the performance of NOVI is  conducted on two real-world large-scale regression datasets: the YearMSD dataset and the Airline dataset. The YearMSD dataset has a large input dimension of D = 90 and a data size of approximately 500,000. The Airline dataset, on the other hand, has an input dimension of D = 8 and a large data size of approximately 2 million. For the YearMSD dataset, we split the data into training and test sets, using the first 463,810 examples as training data and the last 51,725 examples as test data. Similarly, for the Airline dataset, we take the first 700K points for training and next 100K for testing.

In these regression tasks, the performance metric used is the RMSE of the test data. Table \ref{tab:regression-large} presents the results of the test RMSE and the standard deviation over 10 runs. Notably, it can be observed that NOVI generally outperforms SGHMC, DSVI, and IPVI. Particularly for large-scale regression tasks, the performance of NOVI consistently improves with increasing depth. This observation signifies that our NOVI model maintains its superiority even for very large datasets with millions of data points.

\begin{table*}[t]
    %\centering
    \resizebox{\textwidth}{!}{
    \begin{tabular}{|c|c|c|c|c|c|c|c|c|c|c|c|c|c|c|c|c|}
    \hline
    \multicolumn{1}{|c|}{\multirow{1}{*}{\textbf{Data}}} &\multicolumn{1}{c|}{\multirow{1}{*}{\textbf{DSVI 2}}} &\multicolumn{1}{c|}{\multirow{1}{*}{\textbf{DSVI 3}}}
    &\multicolumn{1}{c|}{\multirow{1}{*}{\textbf{DSVI 4}}}
    &\multicolumn{1}{c|}{\multirow{1}{*}{\textbf{DSVI 5}}}
    &\multicolumn{1}{c|}{\multirow{1}{*}{\textbf{SGHMC 2}}}
    &\multicolumn{1}{c|}{\multirow{1}{*}{\textbf{SGHMC 3}}}
    &\multicolumn{1}{c|}{\multirow{1}{*}{\textbf{SGHMC 4}}}
    &\multicolumn{1}{c|}{\multirow{1}{*}{\textbf{SGHMC 5}}}
    &\multicolumn{1}{c|}{\multirow{1}{*}{\textbf{IPVI 2}}}
    &\multicolumn{1}{c|}{\multirow{1}{*}{\textbf{IPVI 3}}}
    &\multicolumn{1}{c|}{\multirow{1}{*}{\textbf{IPVI 4}}}
    &\multicolumn{1}{c|}{\multirow{1}{*}{\textbf{IPVI 5}}}
    &\multicolumn{1}{c|}{\multirow{1}{*}{\textbf{NOVI 2}}}
    &\multicolumn{1}{c|}{\multirow{1}{*}{\textbf{NOVI 3}}}
    &\multicolumn{1}{c|}{\multirow{1}{*}{\textbf{NOVI 4}}}
    &\multicolumn{1}{c|}{\multirow{1}{*}{\textbf{NOVI 5}}}\\
    \hline
    
    % year
    \multicolumn{1}{|c|}{\multirow{1}{*}{\textbf{Year}}}
    % DSVI
    &\multicolumn{1}{c|}{\multirow{1}{*}{9.58 }} 
    &\multicolumn{1}{c|}{\multirow{1}{*}{8.98 }} 
    &\multicolumn{1}{c|}{\multirow{1}{*}{8.93 }} 
    &\multicolumn{1}{c|}{\multirow{1}{*}{8.87 }}
    % SGHMC
    &\multicolumn{1}{c|}{\multirow{1}{*}{9.05 }}
    &\multicolumn{1}{c|}{\multirow{1}{*}{8.91 }}
    &\multicolumn{1}{c|}{\multirow{1}{*}{8.85 }}
    &\multicolumn{1}{c|}{\multirow{1}{*}{8.81 }}
     % IPVI
    &\multicolumn{1}{c|}{\multirow{1}{*}{8.95}}
    &\multicolumn{1}{c|}{\multirow{1}{*}{8.84}}
    &\multicolumn{1}{c|}{\multirow{1}{*}{8.80}}
    &\multicolumn{1}{c|}{\multirow{1}{*}{8.79}}
    % NOVI
    &\multicolumn{1}{c|}{\multirow{1}{*}{\textbf{8.84 }}}
    &\multicolumn{1}{c|}{\multirow{1}{*}{\textbf{8.79} }}
    &\multicolumn{1}{c|}{\multirow{1}{*}{\textbf{8.76} }}
    &\multicolumn{1}{c|}{\multirow{1}{*}{\textbf{8.74} }}\\
    \hline
    
    % Airline
    \multicolumn{1}{|c|}{\multirow{1}{*}{\textbf{Airline}}}
    % DSVI
    &\multicolumn{1}{c|}{\multirow{1}{*}{24.6 }} &\multicolumn{1}{c|}{\multirow{1}{*}{24.3 }} 
    &\multicolumn{1}{c|}{\multirow{1}{*}{24.2}} 
    &\multicolumn{1}{c|}{\multirow{1}{*}{24.1 }}
    % SGHMC
    &\multicolumn{1}{c|}{\multirow{1}{*}{24.1 }}
    &\multicolumn{1}{c|}{\multirow{1}{*}{23.8}}
    &\multicolumn{1}{c|}{\multirow{1}{*}{23.7}}
    &\multicolumn{1}{c|}{\multirow{1}{*}{23.6 }}
    % IPVI
    &\multicolumn{1}{c|}{\multirow{1}{*}{24.0 }}
    &\multicolumn{1}{c|}{\multirow{1}{*}{23.7 }}
    &\multicolumn{1}{c|}{\multirow{1}{*}{23.6 }}
    &\multicolumn{1}{c|}{\multirow{1}{*}{23.6 }}
    % NOVI
    &\multicolumn{1}{c|}{\multirow{1}{*}{\textbf{23.8 }}}
    &\multicolumn{1}{c|}{\multirow{1}{*}{\textbf{23.6} }}
    &\multicolumn{1}{c|}{\multirow{1}{*}{\textbf{23.5} }}
    &\multicolumn{1}{c|}{\multirow{1}{*}{\textbf{23.5}}}\\
    \hline
    \end{tabular}}
    \caption{Regression test RMSE results for large datasets.}
    \label{tab:regression-large}
\end{table*}
\subsection{Image Classification}
We evaluate our method on multiclass classification tasks using the MNIST \cite{lecun1998mnist}, Fashion-MNIST \cite{xiao2017fashion}, and CIFAR-10 \cite{krizhevsky2009learning} datasets. The first two datasets consist of grayscale images of size $28\times28$ pixels, while CIFAR-10 comprises colored images of size $32\times32$ pixels. The results are presented in Table \ref{tab:image-classification}. We note that our method outperforms the other three methods on all three datasets, with significantly less training time and iterations. Additionally, we evaluate our approach using three UCI classification datasets, and the results are presented in Appendix D.A.

 Furthermore, we conduct supplementary experiments to achieve superior performance on the CIFAR-10 dataset. We utilize the convolutional layers of ResNet-20 \cite{he2016deep} as our feature extractor and achieve a remarkable accuracy of 93.56 on the test set under pre-training \cite{wilson2016stochastic}, surpassing the performance of all baseline methods, as detailed in Table \ref{tab:convolutional-cifar10}.
\begin{table}[t]
    \centering
    \begin{tabular}{|c|c|}
        \hline
        \multicolumn{1}{|c|}{\multirow{1}{*}{\textbf{Models}}} &\multicolumn{1}{c|}{\multirow{1}{*}{\textbf{Accuracy (\%)}}} \\

        \hline
        \multicolumn{1}{|c|}{\multirow{1}{*}{\textbf{CNF \cite{yu2021convolutional}}}} &\multicolumn{1}{c|}{\multirow{1}{*}{76.8}} \\

        \hline
        \multicolumn{1}{|c|}{\multirow{1}{*}{\textbf{BDCGP \cite{dutordoir2020bayesian}}}} &\multicolumn{1}{c|}{\multirow{1}{*}{74.6}} \\

        \hline
        \multicolumn{1}{|c|}{\multirow{1}{*}{\textbf{DCGP \cite{blomqvist2020deep}}}} &\multicolumn{1}{c|}{\multirow{1}{*}{75.9}} \\

        \hline
        \multicolumn{1}{|c|}{\multirow{1}{*}{\textbf{DKL\cite{wilson2016stochastic}}}} &\multicolumn{1}{c|}{\multirow{1}{*}{77.0}} \\

        \hline
        \multicolumn{1}{|c|}{\multirow{1}{*}{\textbf{Resnet-20}}} &\multicolumn{1}{c|}{\multirow{1}{*}{91.3}} \\
        \hline
        \multicolumn{1}{|c|}{\multirow{1}{*}{\textbf{Resnet-56}}} &\multicolumn{1}{c|}{\multirow{1}{*}{93.03}}\\
        \hline

        \hline
        \multicolumn{1}{|c|}{\multirow{1}{*}{\textbf{NOVI-DGP}}} &\multicolumn{1}{c|}{\multirow{1}{*}{\textbf{93.56}}}\\
        \hline
        
    \end{tabular}
    \caption{Convolutional results of CIFAR10 dataset compared with baseline deep learning and DGP methods. Our results indicate that our model outperforms ResNet when compared, with only an addition of less than one-tenth of the parameter count.}
    \label{tab:convolutional-cifar10}
\end{table}

\begin{table}
\centering
\normalsize
% \resizebox{\linewidth}{!}{ %< auto-adjusts font size to fill line
\centering
\resizebox{\linewidth}{!}{\begin{tabular}{llcccccc}   % @{}lccccc@{}
\toprule
Type & DSVI 2 & DSVI 3 & DSVI 4 & DSVI 5 \\
\midrule
Time (s) & 0.835 & 0.903 & 0.965 & 1.339 \\
Iteration & 20K & 20K & 20K & 20K \\
\midrule
Type & IPVI 2 & IPVI 3 & IPVI 4 & IPVI 5 \\
\midrule
Time (s) & 0.117 & 0.162 & 0.211 & 0.260 \\
Iteration & 20K & 20K & 20K & 20K \\
\midrule
Type & SGHMC 2 & SGHMC 3 & SGHMC 4 & SGHMC 5 \\
\midrule
Time (s) & 0.630 & 1.000 & 1.490 & 1.870 \\
Iteration & 20K & 20K & 20K & 20K \\
\midrule
Type & NOVI 2 & NOVI 3 & NOVI 4 & NOVI 5 \\
\midrule
Time (s) & 0.391 & 0.613 & 0.863 & 1.123 \\
Iteration & 500 & 500 & 500 & 500 \\
\bottomrule
\end{tabular}}
% \resizebox
\caption{Comparison of training time (s) of a single iteration and total training iterations on Energy dataset. Batch size is set to $1000$ for all four methods.}
\label{tab:computational-complexity}
\end{table}

\subsection{Computational complexity}
We compared the training efficiency of our model with three other methods on a single Tesla V100 GPU using the Energy dataset, and the results are presented in Table \ref{tab:computational-complexity}. It can be observed that our model achieves faster iteration times compared to DSVI and SGHMC and requires less time among the other methods. Furthermore, our method achieves convergence in less than one-tenth of the number of iterations required by the other three methods. Additionally, Table \ref{tab:image-classification} shows that NOVI requires significantly less training time and iterations to converge on high-dimensional image datasets, demonstrating the scalability of our proposed method to larger datasets. Details regarding the number of inducing points used in each method can be found in Appendix D.D.
%\ref{inducing-comparison}.

\subsection{Ablation study} 
To demonstrate the effectiveness of our proposed NOVI method, we compare it with a $2$-layer DGP model by directly maximizing the log-likelihood with randomly initialized $\mathcal{U}$ and hyperparameters $\boldsymbol\nu$. The results are presented in Figure \ref{fig:ablation-regression}. It can be observed that NOVI achieves lower test RMSE and higher train RMSE for all datasets, which indicates that our optimization method reduces overfitting. Although there is some loss fluctuation during the training of our method, it is caused by the unique adversarial training and converges to a stable value after only several hundred iterations. Additional results for ablation study can be found in Appendix.
%\ref{ablation-classification}.

%\section{Related Work}
%\subsection{Hutchinson estimator}
%The computation of the term $\mathrm{Tr}\left(\mathbf{\nabla}_{\boldsymbol{x}}\boldsymbol{\phi }_{\eta}\left( \boldsymbol{x} \right) \right)$ in Equation (\ref{eq:lsd}) takes $\mathcal{O} \left( d \right)$ vector-Jacobian products, which is expensive. Hutchinson estimator requires only one vector-Jacobian product to compute, can be constructed from the following identity by multiplying that matrix by a noise vector twice. \cite{hutchinson1989stochastic}  This single-sample Monte-Carlo estimator, has been widely used in the machine learning communityIn recent years,  \cite{grathwohl2018ffjord,tsitsulin2019shape,han2017approximating}, owing to its efficiency and unbiasedness.
%\begin{equation}
%\mathrm{Tr}\left( \mathbf{\nabla }_{\boldsymbol{x}}\boldsymbol{\phi }_{\eta}\left( x \right) \right) =\mathbb{E} _{\epsilon \sim \mathcal{N} \left( 0,I \right)}\left[ \epsilon ^T\mathbf{\nabla }_{\boldsymbol{x}}\boldsymbol{\phi }_{\eta}\left( \boldsymbol{x} \right) \epsilon \right].
%\end{equation}
%Thus, during training, we use an \textbf{E}fficient version of the LSD objective (\ref{eq:lsd}) with
%\begin{equation}
%\mathrm{LSD}\mathbf{E}\left( q,p;\eta \right) =\;\mathbb{E} _{\boldsymbol{x}\sim q,\epsilon \sim \mathcal{N} \left( 0,I \right)}\left[ \nabla _{\boldsymbol{x}}\log p\left( \boldsymbol{x} \right) ^T\boldsymbol{\phi }_{\eta}\left( \boldsymbol{x} \right) +\epsilon ^T\nabla _{\boldsymbol{x}}\boldsymbol{\phi }_{\eta}\left( \boldsymbol{x} \right) \epsilon \right] .
%\end{equation}

\section{Conclusion}
This paper introduces a novel framework called NOVI, which integrates the Stein Discrepancy with Deep Gaussian Processes (DGPs) to model non-Gaussian and hierarchical-related posteriors, thereby enhancing the flexibility of DGP models. The approach involves generating inducing variables from a neural generator and optimizing them jointly with variational parameters through adversarial training. Theoretical analysis shows that the bias introduced by our method can be bounded by Fisher divergence, enabling efficient optimization of the neural generator.

Empirical evaluation indicates that NOVI outperforms state-of-the-art approximation methods for both regression and classification tasks, while requiring significantly less training time and iterations to converge. We validated our model on 18 publicly available datasets, including 6 classification datasets and 12 regression datasets. These datasets range in sample size from hundreds (e.g., Boston) to millions (e.g., Year), and are mostly from real-world scenarios. Our model outperformed the latest 5 baseline methods on 16 of these datasets. The improvements, particularly in metrics like MSE, are considerable, indicating the effectiveness of our model from a hypothesis testing perspective. We also observed performance fluctuations in some datasets as the number of DGP layers increased during experiments. We hope that future work can address this limitation by making technical breakthroughs in this area.

Future work could also focus on  utilizing Neural Architecture Search (NAS) \cite{elsken2019neural} methods to obtain more suitable network architectures for practical applications. Overall, the proposed NOVI framework represents a significant advancement in the field of deep learning, and holds promise for a wide range of applications in both academia and industry.

%\newpage
\appendix
\appendices
\section{ The Solution to SVGP and DSVI}
%\ref{thm:1}
\label{theorem-1}

Due to the Gaussian mean-field assumptions, the solution to SVGP has an analytical solution
\begin{align}
\label{mu}
\begin{aligned}
q(\mathbf{f})&=\mathcal{N}(\mathbf{f}|\boldsymbol{\mu}, \mathbf{\Sigma})\\ \text{where} \qquad
\boldsymbol{\mu} &= {\mathbf{K}_{\mathbf{X} \mathbf{Z}}\mathbf{K}_{\mathbf{Z}\mathbf{Z}}^{-1}\mathbf{m}} \\ \mathbf{\Sigma} &= {\mathbf{K}_{\mathbf{X}\mathbf{X}} - \mathbf{K}_{\mathbf{X}\mathbf{Z}}\mathbf{K}_{\mathbf{Z}\mathbf{Z}}^{-1}(\mathbf{K}_{\mathbf{Z}\mathbf{Z}} - \mathbf{S})\mathbf{K}_{\mathbf{Z} \mathbf{Z}}^{-1}\mathbf{K}_{\mathbf{Z}\mathbf{X}}} 
\end{aligned}
\end{align}

While performing similarly in DSVI, they have a analytical form for $q(\mathbf{F})$
\begin{small}
\begin{equation}
\begin{split}
&q(\{\mathbf{F}_{\ell}\}_{\ell =1}^{L})=\\    
&\prod\nolimits_{\ell=1}^L{\prod\nolimits_{d=1}^{D_{\ell}}{\int{q\left( \mathbf{F}_{\ell, d}| \mathbf{F}_{\ell -1}, \mathbf{U}_{\ell, d} \right) q\left( \mathbf{U}_{\ell ,d} \right) d\mathbf{U}_{\ell ,d}}}}\\&=\prod\nolimits_{\ell =1}^L{\prod\nolimits_{d=1}^{D_{\ell}}{\mathcal{N} \left( \mathbf{F}_{\ell ,d}|\boldsymbol{\mu }_{\ell ,d},\mathbf{\Sigma }_{\ell, d} \right)}},
\end{split}
\end{equation}
\end{small}
where $\boldsymbol{\mu }_{\ell ,d},\mathbf{\Sigma }_{\ell, d}$ is defined as Equation (\ref{mu}).

\section{Proof of Theorem 1}
\begin{theorem}
Assuming that $\mathcal{U} \in \varOmega$, $\nu \in \varUpsilon$ where $\varOmega$ and $\varUpsilon$ are both compact spaces. We can obtain an asymptotically unbiased estimator for the score function $\nabla_{\mathcal{U}}\log p(\mathcal{U} |\mathcal{D},\nu)$ in Equation (18), which converges in probability to the true value.
%\ref{13}
\begin{equation}
\begin{split}
&\nabla _{\mathcal{U}}\log p(\mathcal{U} |\mathcal{D}, \nu) \approx -\left( \mathbf{\Delta}_1,\dots,\mathbf{\Delta }_{\ell},\dots,\mathbf{\Delta }_L \right)^\top \\&+\mathbf{\nabla}_{\mathcal{U}}\log \sum\nolimits_{s=1}^S{p(\mathbf{y}|\widehat{\mathbf{F}}_{L}^{\left( s \right)})}
\end{split}
\end{equation}
where $\boldsymbol{\Delta}_{\ell}=$\\
\begin{small}$((\mathbf{K}_{\mathbf{Z}_{\ell} \mathbf{Z}_{\ell}}{ }^{-1} \mathbf{U}_{\ell, 1})^\top,...,(\mathbf{K}_{\mathbf{Z}_{\ell} \mathbf{Z}_{\ell}}{ }^{-1} \mathbf{U}_{\ell, d})^\top,...,(\mathbf{K}_{\mathbf{Z}_{\ell} \mathbf{Z}_{\ell}}{ }^{-1} \mathbf{U}_{\ell, D_{\ell}})^\top)$ \end{small} and \begin{small} $\widehat{\mathbf{F}}_{\ell,d}^{( s )} \sim
\mathcal{N} (\mathbf{K}_{\widehat{\mathbf{F}}_{\ell -1}\mathbf{Z}_{\ell}}\mathbf{K}_{\mathbf{Z}_{\ell}\mathbf{Z}_{\ell}}^{-1}\mathbf{U}_{\ell,d},\mathbf{K}_{\widehat{\mathbf{F}}_{\ell-1}\widehat{\mathbf{F}}_{\ell -1}}-\mathbf{K}_{\widehat{\mathbf{F}}_{\ell -1}\mathbf{Z}_{\ell}}\mathbf{K}_{\mathbf{Z}_{\ell}\mathbf{Z}_{\ell}}^{-1}\mathbf{K}_{\mathbf{Z}_{\ell}\widehat{\mathbf{F}}_{\ell -1}}) $ \end{small}  for $\ell = 1 ,\dots, L$, $S$ is the number of samples involved in estimation.
\end{theorem}

\begin{proof} 
From Bayes Formula:
\begin{small}
\begin{align}
\begin{aligned}
\label{eq:bayes}
\log p(\mathcal{U} |\mathcal{D},\nu )&=\log \frac{p( \mathcal{U}) p( \mathcal{D} |\mathcal{U},\nu)}{p( \mathcal{D})}\\&=\log p(\mathcal{U}) +\log p( \mathcal{D} |\mathcal{U},\nu) -\log p(\mathcal{D}), 
\end{aligned}
\end{align}
\end{small}
 
since the prior term $p\left( \mathbf{U}_{\ell,d} \right) =\mathcal{N} \left( 0,\mathbf{K}_{\mathbf{Z}_{\ell}\mathbf{Z}_{\ell}} \right) $, the gradient with $ \mathcal{U}$ is a long vector and is tractable: 
\begin{small}
\begin{equation}
\begin{split}
    &\mathbf{\nabla }_{\mathcal{U}}\log p\left( \mathcal{U} \right)\\
     &=\mathbf{\nabla }_{\mathcal{U}}\log \prod\nolimits_{\ell =1}^L{\prod\nolimits_{d=1}^{D_{\ell}}{p\left( \mathbf{U}_{\ell ,d} \right)}}\\
     &=-\frac{1}{2}\sum\nolimits_{\ell =1}^L{\sum\nolimits_{d=1}^{D_{\ell}}{\mathbf{\nabla }_{\mathcal{U}}\mathbf{U}_{\ell,d}^{T}\mathbf{K}_{\mathbf{Z}_{\ell}\mathbf{Z}_{\ell}}^{-1}\mathbf{U}_{\ell,d}}}\\ 
&=-\left( \mathbf{\Delta }_1,\dots ,\mathbf{\Delta }_{\ell},\dots ,\mathbf{\Delta }_L \right)^\top 
\end{split}
\end{equation}
\end{small}
 The third term of Equation (\ref{eq:bayes}) is a constant w.r.t $\mathcal{U}$. We compute the second data likelihood term $\log p( \mathcal{D} |\mathcal{U} ,\nu )$ using re-parameterization trick and Monte Carlo method over each layer:
\begin{small}
$$
\mathbf{\nabla }_{\mathcal{U}}\log p(\mathcal{D} |\mathcal{U} ,{\nu} )=\mathbf{\nabla }_{\mathcal{U}}\log \int{p(\mathbf{y}|\mathbf{F}_L)\prod_{\ell =1}^L{p}(\mathbf{F}_{\ell}|\mathbf{F}_{\ell -1},\mathbf{U}_{\ell})d\mathbf{F}_{\ell -1}}$$
\begin{equation}
    =\mathbf{\nabla }_{\mathcal{U}}\log  \mathbb{E} _{p(\mathbf{F}_L|\mathcal{U})}\,\,p(\mathbf{y}|\mathbf{F}_L)\approx \mathbf{\nabla }_{\mathcal{U}}
     \log  \sum\nolimits_{s=1}^S{p(\mathbf{y}|\widehat{\mathbf{F}}_{L}^{( s )})}
\end{equation}
\end{small}
The last equation in the above expression can be derived from the following conditions: We denote Monte Carlo estimation $
\frac{1}{S}\sum\nolimits_{s=1}^S{p(\mathbf{y}|\widehat{\mathbf{F}}_{L}^{( s )})}$ as $\tilde{p}\left(\mathbf{y}|\mathcal{U} \right)$ and the true value as $p(\mathbf{y}|\mathcal{U} )$, respectively. By the Central Limit Theorem, $
\frac{\tilde{p}(\mathbf{y}|\mathcal{U} )-p(\mathbf{y}|\mathcal{U} )}{\sqrt{\frac{1}{S}Var\left( p(\mathbf{y}|\widehat{\mathbf{F}}_{L}^{(s)}) \right)}}\sim \mathcal{N} \left( 0,1 \right) 
$, i.e., $\tilde{p}(\mathbf{y}|\mathcal{U}) \xrightarrow[]{\mathcal{P}} p(\mathbf{y}|\mathcal{U})$ and $\nabla \tilde{p}(\mathbf{y}|\mathcal{U}) \xrightarrow[]{\mathcal{P}} \nabla p(\mathbf{y}|\mathcal{U})$, since $
\nabla \tilde{p}(\mathbf{y}|\mathcal{U} )-\nabla p(\mathbf{y}|\mathcal{U} )=\nabla \sqrt{\frac{Var\left( p(\mathbf{y}|\widehat{\mathbf{F}}_{L}^{(s)}) \right)}{S}}\epsilon \xrightarrow[]{\mathcal{P}}0$ as S increases, where $\epsilon \sim \mathcal{N} \left( 0,1 \right)$. 
%Assuming that $\log p\left( \mathbf{y}|\mathcal{U} \right)$ and $\log \tilde{p}\left( \mathbf{y}|\mathcal{U} \right)$ are bounded, since the log likelihood function log $p(\mathbf{y}|\mathcal{U},\nu)$ is  continuous on compact domain, 
The likelihood function  $p(\mathbf{y}|U,\nu)$ is a continuous bounded function defined on a compact domain $\varOmega$ and $\varUpsilon$, then uniform continuity guarantees the boundedness of its derivative, then we have:%we can use the definition of the operator norm of the $\nabla$ operator, denoted by $\left|| \nabla \right||$, and the local Lipschitz property of the logarithm function to bound the difference between $\nabla \log p\left( \mathbf{y}|\mathcal{U} \right)$ and $\nabla \log \tilde{p}\left( \mathbf{y}|\mathcal{U} \right)$ as follows: 
\begin{align}
\begin{array}{l}
\|\nabla \log p(\mathbf{y} \mid \mathcal{U})-\nabla \log \tilde{p}(\mathbf{y} \mid \mathcal{U})\| \\
=\left\|\frac{\nabla p(\mathbf{y} \mid \mathcal{U})}{p(\mathbf{y} \mid \mathcal{U})}-\frac{\nabla \bar{p}(\mathbf{y} \mid \mathcal{U})}{\tilde{p}(\mathbf{y} \mid \mathcal{U})}\right\| \\
=\left\|\frac{\tilde{p}(\mathbf{y} \mid \mathcal{U}) \nabla p(\mathbf{y} \mid \mathcal{U})-p(\mathbf{y} \mid \mathcal{U}) \nabla \tilde{p}(\mathbf{y} \mid \mathcal{U})}{p(\mathbf{y} \mid \mathcal{U}) \tilde{p}(\mathbf{y} \mid \mathcal{U})}\right\| \\
=\left\|\frac{\tilde{p}(\mathbf{y} \mid \mathcal{U}) \nabla p(\mathbf{y} \mid \mathcal{U})-p(\mathbf{y} \mid \mathcal{U}) \nabla p(\mathbf{y} \mid \mathcal{U})+p(\mathbf{y} \mid \mathcal{U}) \nabla p(\mathbf{y} \mid \mathcal{U})-p(\mathbf{y} \mid \mathcal{U}) \nabla \tilde{p}(\mathbf{y} \mid \mathcal{U})}{p(\mathbf{y} \mid \mathcal{U}) \tilde{p}(\mathbf{y} \mid \mathcal{U})}\right\| \\
=\left\|\frac{(\tilde{p}(\mathbf{y} \mid \mathcal{U})-p(\mathbf{y} \mid \mathcal{U})) \nabla p(\mathbf{y} \mid \mathcal{U})+p(\mathbf{y} \mid \mathcal{U})(\nabla p(\mathbf{y} \mid \mathcal{U})-\nabla \tilde{p}(\mathbf{y} \mid \mathcal{U}))}{p(\mathbf{y} \mid \mathcal{U}) \tilde{p}(\mathbf{y} \mid \mathcal{U})}\right\| \\
\leqslant \frac{1}{p(\mathbf{y} \mid \mathcal{U}) \tilde{p}(\mathbf{y} \mid \mathcal{U})} \cdot\left(\begin{array}{c}
\|\nabla p(\mathbf{y} \mid \mathcal{U})\| \cdot\|\tilde{p}(\mathbf{y} \mid \mathcal{U})-p(\mathbf{y} \mid \mathcal{U})\| \\
+p(\mathbf{y} \mid \mathcal{U})\|(\nabla p(\mathbf{y} \mid \mathcal{U})-\nabla \tilde{p}(\mathbf{y} \mid \mathcal{U}))\|
\end{array}\right) \\
\xrightarrow[]{\mathcal{P}}0
\end{array}
\end{align}

%the log likelihood function log $p(\mathbf{y}|U,\nu)$ is a continuous bounded function defined on a compact domain $\varOmega$ and $\varUpsilon$, then its gradient must be bounded and the theorem holds true under this condition. 
It is easy to derive from the above equation that $\nabla \log \tilde{p}(\mathbf{y} \mid \mathcal{U}) \xrightarrow[]{\mathcal{P}} \nabla \log p(\mathbf{y}|\mathcal{U})$, the approximately equal sign means that the right-hand side converges to the left-hand side in probability. From the above expression, we can also conclude that this estimator is asymptotically unbiased.

we draw $S$ samples $\widehat{\mathbf{F}}_{\ell, d}^{(s)}$ from
$\widehat{\mathbf{F}}_{\ell,d}\sim p(\mathbf{F}_{\ell,d}|\widehat{\mathbf{F}}_{\ell-1},\mathbf{U}_{\ell,d} ) $ for $\ell = 1 ,\dots, L$ as
\begin{align}
\begin{aligned}
\widehat{\mathbf{F}}_{\ell,d}&=\mathbf{K}_{\widehat{\mathbf{F}}_{\ell -1}\mathbf{Z}_{\ell}}\mathbf{K}_{\mathbf{Z}_{\ell}\mathbf{Z}_{\ell}}^{-1}\mathbf{U}_{\ell,d}\\&+\epsilon _{\ell}\odot \sqrt{\mathrm{diag}\,(\mathbf{K}_{\widehat{\mathbf{F}}_{\ell -1}\widehat{\mathbf{F}}_{\ell -1}}-\mathbf{K}_{\widehat{\mathbf{F}}_{\ell -1}\mathbf{Z}_{\ell}}\mathbf{K}_{\mathbf{Z}_{\ell}\mathbf{Z}_{\ell}}^{-1}\mathbf{K}_{\mathbf{Z}_{\ell}\widehat{\mathbf{F}}_{\ell -1}})}
\end{aligned}
\end{align}
where $\epsilon ^{\ell}\sim \mathcal{N} ( 0,\mathbf{I}_{D^{\ell}} ) $. As a result, we obtain the score function via automatic differentiation:
\begin{equation}
\begin{split}
&\nabla _{\mathcal{U}}\log p(\mathcal{U} |\mathcal{D}, \boldsymbol{\nu}) \approx -\left( \mathbf{\Delta}_1,\dots,\mathbf{\Delta }_{\ell},\dots,\mathbf{\Delta }_L \right)^\top \\&+\mathbf{\nabla}_{\mathcal{U}}\log \sum\nolimits_{s=1}^S{p(\mathbf{y}|\widehat{\mathbf{F}}_{L}^{\left( s \right)})}
\end{split}
\end{equation}

%Moreover, for regression task, let $S=1$, Equation (\ref{score funca}) has a simpler form:
%\begin{equation}
    %\nabla _{\mathcal{U}}\log p\left( \mathcal{U} |\mathcal{D} ,\nu \right) =-\left( \mathbf{\Delta }_1,\dots ,\mathbf{\Delta }_{\ell},\dots ,\mathbf{\Delta }_L \right) +\frac{1}{\sigma ^2}\mathbf{\nabla }_{\mathcal{U}}\widehat{\mathbf{F}}_{L}^{\left( s \right)T}\left( \mathbf{y}-\widehat{\mathbf{F}}_{L}^{\left( s \right)} \right) 
%\end{equation}
%where $\sigma$ is the noise variance.
\end{proof}

\section{Proof of Theorem 2 and Theorem 3}
%\ref{thm:2} \ref{thm:3}
\label{theorem-23-app}
\begin{definition}
Let $p(\boldsymbol{x})$ be a probability density supported on $\mathcal{X} \subseteq \mathbb{R} ^d$ and $\boldsymbol{\phi }: \mathcal{X} \rightarrow \mathbb{R}^d$ be  a  differentiable function, we define Langevin-Stein Operator (LSO) {\cite{ranganath2016operator}}
\begin{equation}
\mathcal{A}_p\boldsymbol{\phi }( \boldsymbol{x} ) =\nabla_{\boldsymbol{x}} \log p( \boldsymbol{x} ) ^T\boldsymbol{\phi }( \boldsymbol{x} ) +\mathrm{Tr}( \nabla _{\boldsymbol{x}} \boldsymbol{\phi }( \boldsymbol{x})).
\end{equation}
\end{definition}
\begin{lemma}
Let $p(\boldsymbol{x})$ be a probability density function supported on $\mathcal{X} \subseteq \mathbb{R}^d$, and $\boldsymbol{\phi}: \mathcal{X} \rightarrow \mathbb{R}^d$ be a differentiable function. Suppose that $\int_{\partial \mathcal{X}} p(\boldsymbol{x}) \boldsymbol{\phi}(\boldsymbol{x}) d\boldsymbol{x} = \boldsymbol{0}$, where $\partial \mathcal{X}$ represents the boundary of $\mathcal{X}$. Under these conditions, Stein's identity can be expressed as
\begin{align}
\label{equ:steq11}
\mathbb{E} _{\boldsymbol{x}\sim p}\left[ \mathcal{A} _p\boldsymbol{\phi }( \boldsymbol{x} ) \right] =0.
\end{align}
\end{lemma}
\begin{proof}
\begin{equation}
\begin{split}
 \mathbb{E} _{\boldsymbol{x}\sim p}[ \mathcal{A} _p\boldsymbol{\phi }( \boldsymbol{x}) ] &=\mathbb{E} _{\boldsymbol{x}\sim p}[ \nabla _{\boldsymbol{x}}\log p( \boldsymbol{x} ) ^T\boldsymbol{\phi }( \boldsymbol{x} ) +\mathrm{Tr}( \nabla _{\boldsymbol{x}}\boldsymbol{\phi }(\boldsymbol{x}) )]\\ &=\mathrm{Tr}( \mathbb{E} _{\boldsymbol{x}\sim p}[ \boldsymbol{\phi }( \boldsymbol{x} ) \nabla _{\boldsymbol{x}}\log p(\boldsymbol{x})^T+\nabla_{\boldsymbol{x}}\boldsymbol{\phi }(\boldsymbol{x})])   
\end{split}
\end{equation}

\begin{align}
\begin{aligned}
    &\mathbb{E} _{\boldsymbol{x}\sim p}[ \boldsymbol{\phi}(\boldsymbol{x} )\nabla _{\boldsymbol{x}}\log p(\boldsymbol{x}) ^T+\nabla_{\boldsymbol{x}}\boldsymbol{\phi }(\boldsymbol{x} ) ]\\&=\int_{\mathcal{X}}{p( \boldsymbol{x}) \boldsymbol{\phi }( \boldsymbol{x}) \nabla _{\boldsymbol{x}}\log p(\boldsymbol{x} )^T+p(\boldsymbol{x} ) \nabla _{\boldsymbol{x}}\boldsymbol{\phi}(\boldsymbol{x})d\boldsymbol{x}}\\&=\int_{\mathcal{X}}{\nabla_{\boldsymbol{x}}( p(\boldsymbol{x})\boldsymbol{\phi}( \boldsymbol{x} )) d\boldsymbol{x}} 
\end{aligned}
\end{align}

From Divergence Theorem:
\begin{align}
\begin{aligned}
   &\mathrm{Tr(}\int_{\mathcal{X}}{\nabla _{\boldsymbol{x}}(p(\boldsymbol{x})\boldsymbol{\phi }(\boldsymbol{x}))d\boldsymbol{x}})\\&=\int_{\mathcal{X}}{\mathrm{div(}p(\boldsymbol{x})\boldsymbol{\phi }(\boldsymbol{x}))d\boldsymbol{x}}\\&=\int_{\partial \mathcal{X}}{p(\boldsymbol{x})\boldsymbol{\phi }(\boldsymbol{x})^T\boldsymbol{n}(\boldsymbol{x})}d\boldsymbol{x}=0
\end{aligned} 
\end{align}
where $\boldsymbol{n}(\boldsymbol{x})$ is the outward-pointing unit vector on the boundary of $\mathcal{X}$.
\end{proof} 
\begin{lemma} 
\label{lemma2}
Suppose $p(\boldsymbol{x})$, $q(\boldsymbol{x})$ are probability densities supported on $\mathcal{X} \subseteq \mathbb{R}^d$ and $\boldsymbol{\phi }: \mathcal{X} \rightarrow \mathbb{R}^d$ is a differentiable function satisfying $\int_{\partial\mathcal{X}}{p(\boldsymbol{x}) \boldsymbol{\phi }(\boldsymbol{x})}d\boldsymbol{x}=\boldsymbol{0}$ and $\int_{\partial\mathcal{X}}{q(\boldsymbol{x}) \boldsymbol{\phi }(\boldsymbol{x})}d\boldsymbol{x}=\boldsymbol{0}$, then
\begin{equation}
\mathbb{E} _{\boldsymbol{x}\sim q}\left[ \mathcal{A} _p\boldsymbol{\phi }(\boldsymbol{x}) \right] =\mathbb{E} _{\boldsymbol{x}\sim q}[ ( \nabla _{\boldsymbol{x}}\log p(\boldsymbol{x})-\nabla _{\boldsymbol{x}}\log q(\boldsymbol{x}) ) ^T\boldsymbol{\phi }(\boldsymbol{x})]
\end{equation}
\end{lemma}
\begin{proof}
By Lemma \ref{lemma1},
\begin{equation}
\begin{split}
&\mathbb{E} _{\boldsymbol{x}\sim q}\left[ \nabla _{\boldsymbol{x}}\log q(\boldsymbol{x})^T\boldsymbol{\phi }(\boldsymbol{x})+\mathrm{Tr}\left( \nabla _{\boldsymbol{x}}\boldsymbol{\phi }(\boldsymbol{x}) \right) \right] =0\\
&\Rightarrow \mathbb{E} _{\boldsymbol{x}\sim q}\left[ \mathrm{Tr}\left( \nabla _{\boldsymbol{x}}\boldsymbol{\phi }(\boldsymbol{x}) \right) \right] =
-\mathbb{E} _{\boldsymbol{x}\sim q}\left[ \nabla _{\boldsymbol{x}}\log q(\boldsymbol{x})^T\boldsymbol{\phi }(\boldsymbol{x}) \right]     
\end{split}
\end{equation}

\begin{equation}
\begin{split}
\mathbb{E} _{\boldsymbol{x}\sim q}\left[ \mathcal{A} _p\boldsymbol{\phi }(\boldsymbol{x}) \right]&=\mathbb{E} _{\boldsymbol{x}\sim q}[\nabla _{\boldsymbol{x}}\log p(\boldsymbol{x})^T\boldsymbol{\phi }(\boldsymbol{x})+\mathrm{Tr}( \nabla _{\boldsymbol{x}}\boldsymbol{\phi }(\boldsymbol{x}) )]\\&
=\mathbb{E} _{\boldsymbol{x}\sim q}[ ( \nabla _{\boldsymbol{x}}\log p(\boldsymbol{x})-\nabla _{\boldsymbol{x}}\log q(\boldsymbol{x}) ) ^T\boldsymbol{\phi }(\boldsymbol{x}) ]     
\end{split}
\end{equation}
\end{proof}
\begin{lemma}
\label{lemma3}
For any $\boldsymbol{a}, \boldsymbol{y}\in \mathbb{R} ^d$ and $\lambda >0$, the function $\boldsymbol{y}\mapsto \boldsymbol{a}^T\boldsymbol{y}-\lambda \boldsymbol{y}^T\boldsymbol{y}$ achieves its maximum $\frac{1}{4\lambda}\boldsymbol{a}^T\boldsymbol{a}$ if and only if $\boldsymbol{y}=\frac{1}{2\lambda}\boldsymbol{a}$.
\end{lemma}
\begin{proof}
From Cauchy-Schwarz inequality:
\begin{align}
\begin{aligned}
    &\boldsymbol{a}^T\boldsymbol{y}-\lambda \boldsymbol{y}^T\boldsymbol{y}\leqslant \| \boldsymbol{a} \| _2\| \boldsymbol{y} \| _2-\lambda \left\| \boldsymbol{y} \right\| _{2}^{2}\\&=\frac{1}{4\lambda}\| \boldsymbol{a}\| _{2}^{2}-{\lambda}( \|\boldsymbol{y} \| _2-\frac{1}{2\lambda}\| \boldsymbol{a} \| _2 ) ^2
\leqslant \frac{1}{4\lambda}\| \boldsymbol{a}\| _{2}^{2}.
\end{aligned}
\end{align}

The equality holds iff  $\boldsymbol{y}=\frac{1}{2\lambda}\boldsymbol{a}$.
\end{proof}
\begin{definition}
The Fisher divergence {\cite{sriperumbudur2017density}} between two suitably smooth density functions is defined as 
\begin{equation}
FD\left( q,p \right) =\int_{\mathbb{R} ^d}{\left\| \mathbf{\nabla }\log q\left( \theta \right) -\mathbf{\nabla }\log p\left( \theta \right) \right\| _{2}^{2}q\left( \theta \right) d\theta}.
\end{equation}
\end{definition}

\begin{theorem}
Training the generator with the optimal discriminator corresponds to minimizing the Fisher divergence between $q_{\boldsymbol{\theta}}$ and $p$, and the corresponding optimal loss for equation (\ref{eq:l}) is
\begin{equation}
    \mathcal{L} \left( \boldsymbol{\theta} ,\boldsymbol{\nu} \right) =\frac{1}{4\lambda}FD\left( q_{\boldsymbol{\theta}}\left( \mathcal{U} \right) ,p\left( \mathcal{U} |\mathcal{D} ,\boldsymbol{\nu} \right) \right), 
\end{equation}
where $\lambda \in \mathbb{R}^+$ is a regularization strength defined in equation (\ref{rsd}). 
\end{theorem}
\begin{proof}
 Let our loss function be $\mathcal{L} (\theta, \nu)$, by Lemma \ref{lemma2},
\begin{equation}
\begin{aligned}
\begin{aligned}
\mathcal{L}(\theta, \nu) & =\sup _{\eta} \mathbb{E}_{q_{\vartheta}(\mathcal{U})}\left[\mathcal{A}_{p} \boldsymbol{\phi}_{\eta}(\mathcal{U})-\lambda \boldsymbol{\phi}_{\eta}(\mathcal{U})^{T} \boldsymbol{\phi}_{\eta}(\mathcal{U})\right] \\
& =\sup _{\eta} \mathbb{E}_{q_{\vartheta}(\mathcal{U})}\left[\left(\nabla_{\mathcal{U}} \log p(\mathcal{U} \mid \mathcal{D}, \nu)\right.\right. \\
& \left.\left.-\nabla_{\mathcal{U}} q_{\theta}(\mathcal{U})\right)^{T} \boldsymbol{\phi}_{\eta}(\mathcal{U})-\lambda \boldsymbol{\phi}_{\eta}(\mathcal{U})^{T} \boldsymbol{\phi}_{\eta}(\mathcal{U})\right]
\end{aligned}
\end{aligned}
\end{equation}

According to Lemma \ref{lemma3}, the above equation attains its maximum value when the function $\boldsymbol{\phi }_{\eta}\left( \mathcal{U} \right) =\nabla _\mathcal{U}\log p(\mathcal{U} |\mathcal{D} ,\nu )-\nabla _\mathcal{U}q_{\theta}\left( \mathcal{U} \right)$,

\begin{equation}
\begin{split}
    \mathcal{L} (\theta,\nu ) &=\frac{1}{4\lambda}\mathbb{E} _{q_{\theta}\left( \mathcal{U} \right)}[\left\| \nabla _\mathcal{U}\log p(\mathcal{U} |\mathcal{D} ,\nu )-\nabla _\mathcal{U}q_{\theta}\left( \mathcal{U} \right) \right\| _{2}^{2} ]\\&=\frac{1}{4\lambda}FD\left( q_{\theta}\left( \mathcal{U} \right) ,p\left( \mathcal{U} |\mathcal{D} ,\nu \right) \right)
\end{split}
\end{equation}
The optimal discriminator is:
\begin{equation}
\boldsymbol{\phi }_{\eta ^\star}(\mathcal{U})=\frac{1}{2\lambda}\left( \nabla _\mathcal{U}\log p(\mathcal{U} |\mathcal{D} ,\nu )-\nabla_\mathcal{U} q_{\theta}\left( \mathcal{U} \right) \right)
\end{equation}
\end{proof}

\begin{lemma}
\label{lemma4}
Suppose $p(\boldsymbol{x})$, $q(\boldsymbol{x})$ are probability densities  on $ \mathbb{R} ^d$ and $\boldsymbol{\phi }: \mathbb{R} ^d \rightarrow \mathbb{R}^d$ is a  differentiable function that satisfies  $\underset{\left\| \boldsymbol{x} \right\| \rightarrow \infty}{\lim}\,\,q\left( \boldsymbol{x} \right) \boldsymbol{\phi }\left( \boldsymbol{x} \right) =\boldsymbol{0}$, we have 
\begin{equation}
\left| \mathbb{E} _{\boldsymbol{x}\sim q}\left[ \mathcal{A} _p\boldsymbol{\phi }(\boldsymbol{x}) \right] \right|\leqslant \sqrt{\mathbb{E} _{\boldsymbol{x}\sim q}\left\| \boldsymbol{\phi }(\boldsymbol{x}) \right\| _{2}^{2}}\sqrt{FD\left( q,p \right)}
\end{equation}
\end{lemma}
\begin{proof}
By Lemma \ref{lemma2}, we have:
\begin{small}
\begin{equation}
\left| \mathbb{E} _{\boldsymbol{x}\sim q}\left[ \mathcal{A} _p\boldsymbol{\phi }(\boldsymbol{x}) \right] \right|=| \mathbb{E} _{\boldsymbol{x}\sim q}[ \left( \mathbf{\nabla }_{\boldsymbol{x}}\log p\left( x \right) -\mathbf{\nabla }_{\boldsymbol{x}}\log q\left( x \right) \right) ^T\boldsymbol{\phi }(\boldsymbol{x}) ]|.
\end{equation}
\end{small}

From Cauchy-Schwarz inequality and Hölder's inequality:
\begin{align}
\begin{aligned}
 &|\mathbb{E}_{\boldsymbol{x}\sim q}[ \left( \mathbf{\nabla}_{\boldsymbol{x}}\log p\left( x \right) -\mathbf{\nabla }_{\boldsymbol{x}}\log q\left( x \right) \right)^T\boldsymbol{\phi }(\boldsymbol{x}) ]|\\& \leqslant \mathbb{E}_{\boldsymbol{x}\sim q}\left[ \left\| \mathbf{\nabla }_{\boldsymbol{x}}\log p\left( x \right) -\mathbf{\nabla }_{\boldsymbol{x}}\log q\left( x \right) \right\| _2\left\| \boldsymbol{\phi }(\boldsymbol{x}) \right\| _2 \right]\\&\leqslant \sqrt{\mathbb{E} _{\boldsymbol{x}\sim q}\left\| \mathbf{\nabla }_{\boldsymbol{x}}\log p\left( x \right) -\mathbf{\nabla }_{\boldsymbol{x}}\log q\left( x \right) \right\| _{2}^{2}}\sqrt{\mathbb{E} _{\boldsymbol{x}\sim q}\left\| \boldsymbol{\phi }(\boldsymbol{x}) \right\| _{2}^{2}} \\&= \sqrt{\mathbb{E} _{\boldsymbol{x}\sim q}\left\| \boldsymbol{\phi }(\boldsymbol{x}) \right\| _{2}^{2}}\sqrt{FD\left( q,p \right)}
\end{aligned}   
\end{align}
\end{proof}

\begin{definition}
\label{mydef3}
Suppose $p(\boldsymbol{x})$ is probability densities on $\mathbb{R}^d$ and $\psi: \mathbb{R}^d\rightarrow \mathbb{R}$ is a function, we define $\boldsymbol{\phi }_{\psi}^{p}\left( \boldsymbol{x} \right)$ as a solution of the Stein equation $\mathcal{A}_p\boldsymbol{\phi }\left( \boldsymbol{x} \right) =\psi(\boldsymbol{x} )-\mathbb{E} _{\boldsymbol{x}\sim p}[\psi(\boldsymbol{x})]$.
\end{definition}
\begin{lemma}
\label{lemma5}
Suppose $\psi: \mathbb{R} ^d\rightarrow \mathbb{R}$ is a bounded function, there exists a bounded solution of the Stein equation.
\end{lemma}
\begin{proof}
Let $h(\boldsymbol{x})=\psi(\boldsymbol{x})-\mathbb{E}_{\boldsymbol{x}\sim p}[ \psi(\boldsymbol{x})]$, $h(\boldsymbol{x})$ is obviously bounded, then
\begin{equation}
\begin{aligned}
    &\phi _1\left( \boldsymbol{x} \right) =\frac{1}{p\left( \boldsymbol{x} \right)}\int_{-\infty}^{x_1}{p\left( t,x_2,...,x_d \right) h\left( t,x_2,...,x_d \right)}dt,\\& \phi _2\left( \boldsymbol{x} \right) =\cdots =\phi _d\left( \boldsymbol{x} \right) =0
\end{aligned}
\end{equation}
is a bounded solution. 
\end{proof}
\begin{lemma}
\label{lemma6}
Suppose $p(\boldsymbol{x})$, $q(\boldsymbol{x})$ are probability densities on $\mathbb{R} ^d$ and $\boldsymbol{\psi}: \mathbb{R}^d\rightarrow \mathbb{R}^n$ is a bounded function. $\forall i\in ( 1,\dots, n)$, let $\boldsymbol{\phi }_{\psi_i}^{p}\left( \boldsymbol{x} \right)$ be a solution of the Stein equation, then we have 
\begin{equation}
\left\| \mathbb{E} _{\boldsymbol{x}\sim q}[\boldsymbol{\psi }\left( \boldsymbol{x} \right) ]-\mathbb{E} _{\boldsymbol{x}\sim p}[\boldsymbol{\psi }\left( \boldsymbol{x} \right) ] \right\| _2\leqslant c_{\boldsymbol{\psi }}^{p,q}\sqrt{FD\left( q,p \right)}
\end{equation}
where $c_{\boldsymbol{\psi }}^{p,q}\triangleq\sqrt{\sum\nolimits_{i=1}^n{\mathbb{E} _{\boldsymbol{x}\sim q}\left\| \boldsymbol{\phi }_{\psi _i}^{p}\left( \boldsymbol{x} \right) \right\| _{2}^{2}}}$ is bounded.
\end{lemma}
\begin{proof}
By Lemma \ref{lemma4}, we have 
\begin{equation}
\begin{aligned}
    &\left| \mathbb{E} _{\boldsymbol{x}\sim q}\left[ \psi_i(\boldsymbol{x}) \right] -\mathbb{E} _{\boldsymbol{x}\sim p}\left[ \psi_i(\boldsymbol{x}) \right] \right|\\&=\left| \mathbb{E} _{\boldsymbol{x}\sim q}\left[ \psi_i(\boldsymbol{x})-\mathbb{E} _{\boldsymbol{x}\sim p}\left[\psi_i(\boldsymbol{x}) \right] \right] \right|\\&=\left| \mathbb{E} _{\boldsymbol{x}\sim q}\left[ \mathcal{A} _p\boldsymbol{\phi }_{\psi_i}^{p}\left( \boldsymbol{x} \right) \right] \right| \leqslant \sqrt{\mathbb{E} _{\boldsymbol{x}\sim q}\left\| \boldsymbol{\phi }_{\psi_i}^{p}\left( \boldsymbol{x} \right) \right\| _{2}^{2}}\sqrt{FD\left( q,p \right)}.
\end{aligned}
\end{equation}

As a result, 
\begin{equation}
\begin{aligned}
    &\left\| \mathbb{E} _{\boldsymbol{x}\sim q}\left[ \boldsymbol{\psi}(\boldsymbol{x}) \right] -\mathbb{E} _{\boldsymbol{x}\sim p}\left[ \boldsymbol{\psi}(\boldsymbol{x}) \right] \right\| _2 \\&= \sqrt{\sum\nolimits_{i=1}^n{\left| \mathbb{E} _{\boldsymbol{x}\sim q}\left[ \psi_i(\boldsymbol{x}) \right] -\mathbb{E} _{\boldsymbol{x}\sim p}\left[ \psi_i(\boldsymbol{x}) \right] \right|^2}}\\&\leqslant \sqrt{\sum\nolimits_{i=1}^n{\mathbb{E} _{\boldsymbol{x}\sim q}\left\| \boldsymbol{\phi }_{x_i}^{p}\left( \boldsymbol{x} \right) \right\|}_{2}^{2}FD\left( q,p \right)}= c_{\boldsymbol{\psi}}^{p,q}\sqrt{FD\left( q,p \right)},
\end{aligned}
\end{equation} 
where 
\begin{equation}
c_{\boldsymbol{\psi}}^{p,q}\triangleq \sqrt{\sum\nolimits_{i=1}^n{\mathbb{E} _{\boldsymbol{x}\sim q}\left\| \boldsymbol{\phi }_{\psi_i}^{p}\left( \boldsymbol{x} \right) \right\| _{2}^{2}}}\leqslant \sqrt{\sum\nolimits_{i=1}^n{\left\| \boldsymbol{\phi }_{\psi_i}^{p}\left( \boldsymbol{x} \right) \right\| _{\infty}^{2}}}
\end{equation} 
is bounded by Lemma \ref{lemma5}.
\end{proof}

\begin{theorem}
The bias of the estimate of the prediction $\widehat{\mathbf{F}}_{L}^\star$ in Equation (21) from the DGPs exact evaluation can be
%\ref{pre}
bounded by the square root of the Fisher divergence between $q_{\theta}( \mathcal{U})$ and $p\left( \mathcal{U} |\mathcal{D} ,\nu \right)$  up to multiplying a constant.
\end{theorem}
\begin{proof}
From Law of Large Numbers, we have 
\begin{equation}
\mathbf{\hat{F}}_{L}^\star=\frac{1}{S}\sum\nolimits_{s=1}^S{\mathbf{\hat{F}}_{L}^{\star\left( s \right)}}\approx \mathbb{E} _{q\left( \mathbf{F}_{L}^\star \right)}[\mathbf{F}_{L}^\star]
\end{equation},
where $S$ denotes the number of samples involved in the estimation and $q\left( \mathbf{F}_{L}^\star \right)$ is represented as:
\begin{small}
\begin{equation}
q(\mathbf{F}_{L}^{\star})=\int{\prod\nolimits_{\ell =1}^L{\prod\nolimits_{d=1}^{D_{\ell}}{p(\mathbf{F}_{\ell ,d}^{\star}|\mathbf{F}_{\ell -1}^{\star},\mathbf{U}_{\ell,d})q_{\theta ^{\star}}\left( \mathbf{U}_{\ell} \right) d\mathbf{F}_{\ell -1}^{\star}d\mathbf{U}_{\ell ,d}}}}.
\end{equation}
\end{small}
The DGPs exact evaluation can be written as:
\begin{equation}
\widetilde{\mathbf{F}}_{L}^\star=\mathbb{E} _{p( \mathbf{F}_{L}^\star|\mathcal{D} ,\nu)}[\mathbf{F}_{L}^{*}].
\end{equation}
Similarly:
\begin{small}
\begin{equation}
\begin{aligned}
    &p(\mathbf{F}_{L}^{\star}|\mathcal{D} ,\nu )\\&=\int{\prod\nolimits_{\ell =1}^L{\prod\nolimits_{d=1}^{D\ell}{p(\mathbf{F}_{\ell}^{\star}\mid \mathbf{F}_{\ell -1}^{\star},\mathbf{U}_{\ell ,d})p\left( \mathbf{U}_{\ell}|\mathcal{D} ,\nu \right) d\mathbf{F}_{\ell-1}^{\star}d\mathbf{U}_{\ell ,d}}}}.
\end{aligned}
\end{equation}
\end{small}

By Lemma \ref{lemma6}:
\begin{small}
\begin{equation}
\begin{split}
&\left\| \widehat{\mathbf{F}}_{L}^{\star}-\widetilde{\mathbf{F}}_{L}^{\star} \right\| _2\\&=\parallel \mathbb{E} _{q(\mathbf{F}_{L}^{\star})}[\mathbf{F}_{L}^{\star}]-\mathbb{E} _{p(\mathbf{F}_{L}^{\star}|\mathcal{D} ,\nu )}[\mathbf{F}_{L}^{\star}]\parallel _2\\
&=\parallel \mathbb{E} _{q(\mathcal{U} )}[\int{\mathbf{F}_{L}^{\star}\prod\nolimits_{\ell =1}^L{\prod\nolimits_{d=1}^{D_{\ell}}{p(\mathbf{F}_{\ell ,d}^{\star}|\mathbf{F}_{\ell -1}^{\star},\mathbf{U}_{\ell ,d})d\mathbf{F}_{\ell -1}^{\star}d\mathbf{F}_{L}^{\star}}}}]\\&-\mathbb{E} _{p(\mathcal{U} |\mathcal{D} ,\nu )}[\int{\mathbf{F}_{L}^{\star}\prod\nolimits_{\ell =1}^L{\prod\nolimits_{d=1}^{D_{\ell}}{p(\mathbf{F}_{\ell ,d}^{\star}|\mathbf{F}_{\ell -1}^{\star},\mathbf{U}_{\ell ,d})d\mathbf{F}_{\ell -1}^{\star}d\mathbf{F}_{L}^{\star}}}}]\parallel _2
\\
&=\left\| \mathbb{E} _{q(\mathcal{U} )}\left[ \boldsymbol{\psi }\left( \mathcal{U} \right) \right] -\mathbb{E} _{p(\mathcal{U} |\mathcal{D} ,\nu )}\left[ \boldsymbol{\psi }\left( \mathcal{U} \right) \right] \right\| _2\\&\leqslant c_{\boldsymbol{\psi }}^{p,q}\sqrt{FD\left( q\left( \mathcal{U} \right) ,p\left( \mathcal{U} |\mathcal{D} ,\nu \right) \right)},
\end{split}
\end{equation}
\end{small}

Since $\varOmega$ and $\varUpsilon$ are both compact, $\boldsymbol{\psi }\left( \mathcal{U} \right) =\int{\mathbf{F}_{L}^{\star}\prod\nolimits_{\ell=1}^L{\prod\nolimits_{d=1}^{D_{\ell}}{p(\mathbf{F}_{\ell, d}^{\star} |\mathbf{F}_{\ell-1}^{\star},\mathbf{U}_{\ell ,d})d\mathbf{F}_{\ell-1}^{\star}d\mathbf{F}_{L}^{\star}}}}$ is obviously bounded.

\end{proof}

\section{Additional Results}
\label{additional-results}

\begin{figure*}
    \centering
    \includegraphics[width=0.99\linewidth]{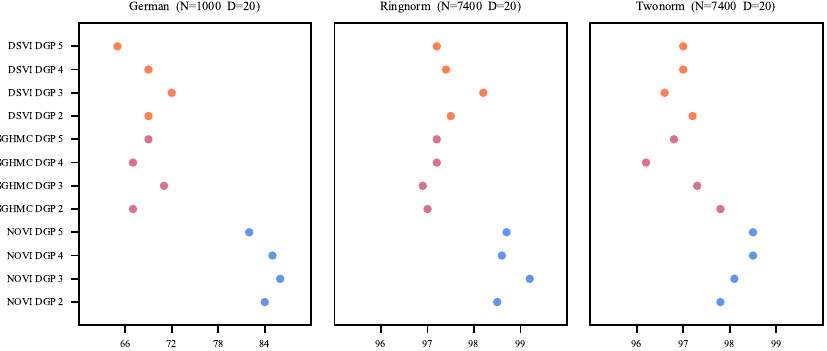}
    \caption{Classification mean test accuracy ($\%$) by our NOVI method (blue), SGHMC (pink) and DSVI (orange) for DGPs on three UCI benchmark datasets. Higher is better.}
    \label{fig:uci-classification}
\end{figure*}
\subsection{UCI Classification Benchmark}
%\label{uci-classification}
We performed classification task on three UCI benchmark datasets, with size ranging from $1000$ to $7400$. Results are reported in Figure \ref{fig:uci-classification} compared through test accuracy as performance metric. It can be observed that NOVI achieves the best results in different sizes of datasets and shows competitive performance within different layers.

\begin{figure*}[t]
    \centering
    \includegraphics[width=0.99\linewidth]{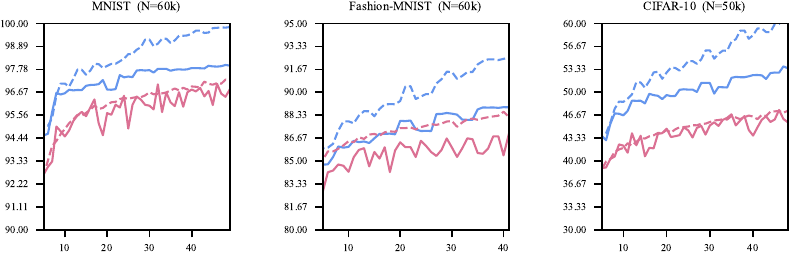}
    \caption{The mean accuracy comparison of NOVI (blue) with Monta Carlo log-likelihood maximization method (pink) using $3$-layer DGP model on three image classification datasets. The results of the training and test sets are shown by dashed and solid lines, respectively.}
    \label{fig:ablation-classification}
\end{figure*}

\begin{table*}[t]
    %\centering
    \resizebox{\textwidth}{!}{
    \begin{tabular}{|c|c|c|c|c|c|c|c|c|c|c|c|c|c|c|c|c|}
    \hline
    \multicolumn{1}{|c|}{\multirow{1}{*}{\textbf{Data}}} &\multicolumn{1}{c|}{\multirow{1}{*}{\textbf{DSVI 2}}} &\multicolumn{1}{c|}{\multirow{1}{*}{\textbf{DSVI 3}}}
    &\multicolumn{1}{c|}{\multirow{1}{*}{\textbf{DSVI 4}}}
    &\multicolumn{1}{c|}{\multirow{1}{*}{\textbf{DSVI 5}}}
    &\multicolumn{1}{c|}{\multirow{1}{*}{\textbf{SGHMC 2}}}
    &\multicolumn{1}{c|}{\multirow{1}{*}{\textbf{SGHMC 3}}}
    &\multicolumn{1}{c|}{\multirow{1}{*}{\textbf{SGHMC 4}}}
    &\multicolumn{1}{c|}{\multirow{1}{*}{\textbf{SGHMC 5}}}
    &\multicolumn{1}{c|}{\multirow{1}{*}{\textbf{IPVI 2}}}
    &\multicolumn{1}{c|}{\multirow{1}{*}{\textbf{IPVI 3}}}
    &\multicolumn{1}{c|}{\multirow{1}{*}{\textbf{IPVI 4}}}
    &\multicolumn{1}{c|}{\multirow{1}{*}{\textbf{IPVI 5}}}
    &\multicolumn{1}{c|}{\multirow{1}{*}{\textbf{NOVI 2}}}
    &\multicolumn{1}{c|}{\multirow{1}{*}{\textbf{NOVI 3}}}
    &\multicolumn{1}{c|}{\multirow{1}{*}{\textbf{NOVI 4}}}
    &\multicolumn{1}{c|}{\multirow{1}{*}{\textbf{NOVI 5}}}\\
    \hline
    
    % Boston
    \multicolumn{1}{|c|}{\multirow{1}{*}{\textbf{Boston}}}
    % DSVI
    &\multicolumn{1}{c|}{\multirow{1}{*}{0.32 (0.02)}} &\multicolumn{1}{c|}{\multirow{1}{*}{0.32 (0.02)}} 
    &\multicolumn{1}{c|}{\multirow{1}{*}{0.32 (0.02)}} 
    &\multicolumn{1}{c|}{\multirow{1}{*}{0.32 (0.02)}}
    % SGHMC
    &\multicolumn{1}{c|}{\multirow{1}{*}{0.37 (0.07)}}
    &\multicolumn{1}{c|}{\multirow{1}{*}{0.38 (0.08)}}
    &\multicolumn{1}{c|}{\multirow{1}{*}{0.35 (0.09)}}
    &\multicolumn{1}{c|}{\multirow{1}{*}{0.39 (0.07)}}
     % IPVI
    &\multicolumn{1}{c|}{\multirow{1}{*}{0.35 (0.06)}}
    &\multicolumn{1}{c|}{\multirow{1}{*}{0.34 (0.05)}}
    &\multicolumn{1}{c|}{\multirow{1}{*}{0.33 (0.06)}}
    &\multicolumn{1}{c|}{\multirow{1}{*}{0.32 (0.04)}}
    % NOVI
    &\multicolumn{1}{c|}{\multirow{1}{*}{\textbf{0.20 (0.01)}}}
    &\multicolumn{1}{c|}{\multirow{1}{*}{0.34 (0.02)}}
    &\multicolumn{1}{c|}{\multirow{1}{*}{0.40 (0.03)}}
    &\multicolumn{1}{c|}{\multirow{1}{*}{0.38 (0.02)}}\\
    \hline
    
    % Energy
    \multicolumn{1}{|c|}{\multirow{1}{*}{\textbf{Energy}}}
    % DSVI
    &\multicolumn{1}{c|}{\multirow{1}{*}{0.05 (0.00)}} &\multicolumn{1}{c|}{\multirow{1}{*}{0.05 (0.00)}} 
    &\multicolumn{1}{c|}{\multirow{1}{*}{0.05 (0.00)}} 
    &\multicolumn{1}{c|}{\multirow{1}{*}{0.05 (0.00)}}
    % SGHMC
    &\multicolumn{1}{c|}{\multirow{1}{*}{0.13 (0.01)}}
    &\multicolumn{1}{c|}{\multirow{1}{*}{0.13 (0.01)}}
    &\multicolumn{1}{c|}{\multirow{1}{*}{0.09 (0.04)}}
    &\multicolumn{1}{c|}{\multirow{1}{*}{0.13 (0.01)}}
    % IPVI
    &\multicolumn{1}{c|}{\multirow{1}{*}{0.13 (0.01)}}
    &\multicolumn{1}{c|}{\multirow{1}{*}{0.13 (0.01)}}
    &\multicolumn{1}{c|}{\multirow{1}{*}{0.12 (0.03)}}
    &\multicolumn{1}{c|}{\multirow{1}{*}{0.11 (0.04)}}
    % NOVI
    &\multicolumn{1}{c|}{\multirow{1}{*}{\textbf{0.04 (0.00)}}}
    &\multicolumn{1}{c|}{\multirow{1}{*}{0.05 (0.00)}}
    &\multicolumn{1}{c|}{\multirow{1}{*}{0.06 (0.00)}}
    &\multicolumn{1}{c|}{\multirow{1}{*}{0.06 (0.00)}}\\
    \hline
    
    % Power
    \multicolumn{1}{|c|}{\multirow{1}{*}{\textbf{Power}}}
    % DSVI
    &\multicolumn{1}{c|}{\multirow{1}{*}{0.22 (0.00)}} &\multicolumn{1}{c|}{\multirow{1}{*}{0.22 (0.00)}} 
    &\multicolumn{1}{c|}{\multirow{1}{*}{0.22 (0.00)}} 
    &\multicolumn{1}{c|}{\multirow{1}{*}{0.22 (0.00)}}
    % SGHMC
    &\multicolumn{1}{c|}{\multirow{1}{*}{0.23 (0.00)}}
    &\multicolumn{1}{c|}{\multirow{1}{*}{0.22 (0.00)}}
    &\multicolumn{1}{c|}{\multirow{1}{*}{0.22 (0.00)}}
    &\multicolumn{1}{c|}{\multirow{1}{*}{0.22 (0.00)}}
    % IPVI
    &\multicolumn{1}{c|}{\multirow{1}{*}{0.22 (0.01)}}
    &\multicolumn{1}{c|}{\multirow{1}{*}{0.22 (0.01)}}
    &\multicolumn{1}{c|}{\multirow{1}{*}{0.22 (0.01)}}
    &\multicolumn{1}{c|}{\multirow{1}{*}{0.21 (0.01)}}
    % NOVI
    &\multicolumn{1}{c|}{\multirow{1}{*}{0.22 (0.00)}}
    &\multicolumn{1}{c|}{\multirow{1}{*}{\textbf{0.21 (0.00)}}}
    &\multicolumn{1}{c|}{\multirow{1}{*}{\textbf{0.21 (0.00)}}}
    &\multicolumn{1}{c|}{\multirow{1}{*}{\textbf{0.21 (0.00)}}}\\
    \hline
    
    % Concrete
    \multicolumn{1}{|c|}{\multirow{1}{*}{\textbf{Concrete}}}
    % DSVI
    &\multicolumn{1}{c|}{\multirow{1}{*}{0.34 (0.01)}} &\multicolumn{1}{c|}{\multirow{1}{*}{0.34 (0.01)}} 
    &\multicolumn{1}{c|}{\multirow{1}{*}{0.34 (0.01)}} 
    &\multicolumn{1}{c|}{\multirow{1}{*}{0.34 (0.01)}}
    % SGHMC
    &\multicolumn{1}{c|}{\multirow{1}{*}{0.35 (0.03)}}
    &\multicolumn{1}{c|}{\multirow{1}{*}{0.33 (0.03)}}
    &\multicolumn{1}{c|}{\multirow{1}{*}{0.31 (0.02)}}
    &\multicolumn{1}{c|}{\multirow{1}{*}{0.31 (0.02)}}
    % IPVI
    &\multicolumn{1}{c|}{\multirow{1}{*}{0.32 (0.02)}}
    &\multicolumn{1}{c|}{\multirow{1}{*}{0.30 (0.03)}}
    &\multicolumn{1}{c|}{\multirow{1}{*}{0.31 (0.03)}}
    &\multicolumn{1}{c|}{\multirow{1}{*}{0.30 (0.04)}}
    % NOVI
    &\multicolumn{1}{c|}{\multirow{1}{*}{0.24 (0.00)}}
    &\multicolumn{1}{c|}{\multirow{1}{*}{0.25 (0.00)}}
    &\multicolumn{1}{c|}{\multirow{1}{*}{0.24 (0.00)}}
    &\multicolumn{1}{c|}{\multirow{1}{*}{\textbf{0.23 (0.00)}}}\\
    \hline
    
    % Yacht
    \multicolumn{1}{|c|}{\multirow{1}{*}{\textbf{Yacht}}}
    % DSVI
    &\multicolumn{1}{c|}{\multirow{1}{*}{0.07 (0.00)}} &\multicolumn{1}{c|}{\multirow{1}{*}{0.07 (0.00)}} 
    &\multicolumn{1}{c|}{\multirow{1}{*}{0.07 (0.00)}} 
    &\multicolumn{1}{c|}{\multirow{1}{*}{0.07 (0.00)}}
    % SGHMC
    &\multicolumn{1}{c|}{\multirow{1}{*}{0.03 (0.01)}}
    &\multicolumn{1}{c|}{\multirow{1}{*}{0.03 (0.01)}}
    &\multicolumn{1}{c|}{\multirow{1}{*}{\textbf{0.02 (0.01)}}}
    &\multicolumn{1}{c|}{\multirow{1}{*}{0.03 (0.01)}}
    % IPVI
    &\multicolumn{1}{c|}{\multirow{1}{*}{0.03 (0.02)}}
    &\multicolumn{1}{c|}{\multirow{1}{*}{0.03 (0.02)}}
    &\multicolumn{1}{c|}{\multirow{1}{*}{0.04 (0.03)}}
    &\multicolumn{1}{c|}{\multirow{1}{*}{0.03 (0.01)}}
    % NOVI
    &\multicolumn{1}{c|}{\multirow{1}{*}{0.03 (0.00)}}
    &\multicolumn{1}{c|}{\multirow{1}{*}{0.09 (0.01)}}
    &\multicolumn{1}{c|}{\multirow{1}{*}{0.08 (0.00)}}
    &\multicolumn{1}{c|}{\multirow{1}{*}{0.06 (0.00)}}\\
    \hline
    % Qsar
    \multicolumn{1}{|c|}{\multirow{1}{*}{\textbf{Qsar}}}
    % DSVI
    &\multicolumn{1}{c|}{\multirow{1}{*}{0.57 (0.00)}} &\multicolumn{1}{c|}{\multirow{1}{*}{0.50 (0.00)}} 
    &\multicolumn{1}{c|}{\multirow{1}{*}{0.47 (0.00)}} 
    &\multicolumn{1}{c|}{\multirow{1}{*}{\textbf{0.42 (0.00)}}}
    % SGHMC
    &\multicolumn{1}{c|}{\multirow{1}{*}{0.56 (0.00)}}
    &\multicolumn{1}{c|}{\multirow{1}{*}{0.56 (0.00)}}
    &\multicolumn{1}{c|}{\multirow{1}{*}{0.56 (0.00)}}
    &\multicolumn{1}{c|}{\multirow{1}{*}{0.56 (0.00)}}
    % IPVI
    &\multicolumn{1}{c|}{\multirow{1}{*}{0.56 (0.01)}}
    &\multicolumn{1}{c|}{\multirow{1}{*}{0.54 (0.01)}}
    &\multicolumn{1}{c|}{\multirow{1}{*}{0.54 (0.01)}}
    &\multicolumn{1}{c|}{\multirow{1}{*}{0.54 (0.01)}}
    % NOVI
    &\multicolumn{1}{c|}{\multirow{1}{*}{0.51 (0.00)}}
    &\multicolumn{1}{c|}{\multirow{1}{*}{0.46 (0.01)}}
    &\multicolumn{1}{c|}{\multirow{1}{*}{0.45 (0.01)}}
    &\multicolumn{1}{c|}{\multirow{1}{*}{0.44 (0.01)}}\\
    \hline
    
    % Protein
    \multicolumn{1}{|c|}{\multirow{1}{*}{\textbf{Protein}}}
    % DSVI
    &\multicolumn{1}{c|}{\multirow{1}{*}{0.81 (0.00)}} &\multicolumn{1}{c|}{\multirow{1}{*}{0.77 (0.00)}} 
    &\multicolumn{1}{c|}{\multirow{1}{*}{0.79 (0.00)}} 
    &\multicolumn{1}{c|}{\multirow{1}{*}{0.73 (0.00)}}
    % SGHMC
    &\multicolumn{1}{c|}{\multirow{1}{*}{0.72 (0.01)}}
    &\multicolumn{1}{c|}{\multirow{1}{*}{0.71 (0.01)}}
    &\multicolumn{1}{c|}{\multirow{1}{*}{0.70 (0.01)}}
    &\multicolumn{1}{c|}{\multirow{1}{*}{0.69 (0.00)}}
    % IPVI
    &\multicolumn{1}{c|}{\multirow{1}{*}{0.68 (0.01)}}
    &\multicolumn{1}{c|}{\multirow{1}{*}{0.65 (0.01)}}
    &\multicolumn{1}{c|}{\multirow{1}{*}{0.65 (0.01)}}
    &\multicolumn{1}{c|}{\multirow{1}{*}{0.65 (0.01)}}
    % NOVI
    &\multicolumn{1}{c|}{\multirow{1}{*}{0.67 (0.00)}}
    &\multicolumn{1}{c|}{\multirow{1}{*}{\textbf{0.65 (0.00)}}}
    &\multicolumn{1}{c|}{\multirow{1}{*}{0.66 (0.00)}}
    &\multicolumn{1}{c|}{\multirow{1}{*}{0.66 (0.00)}}\\
    \hline
    
    % Kin8nm
    \multicolumn{1}{|c|}{\multirow{1}{*}{\textbf{Kin8nm}}}
    % DSVI
    &\multicolumn{1}{c|}{\multirow{1}{*}{0.39 (0.00)}} &\multicolumn{1}{c|}{\multirow{1}{*}{0.37 (0.00)}} 
    &\multicolumn{1}{c|}{\multirow{1}{*}{0.34 (0.00)}} 
    &\multicolumn{1}{c|}{\multirow{1}{*}{0.30 (0.00)}}
    % SGHMC
    &\multicolumn{1}{c|}{\multirow{1}{*}{0.26 (0.01)}}
    &\multicolumn{1}{c|}{\multirow{1}{*}{0.25 (0.01)}}
    &\multicolumn{1}{c|}{\multirow{1}{*}{0.25 (0.01)}}
    &\multicolumn{1}{c|}{\multirow{1}{*}{0.24 (0.01)}}
    % IPVI
    &\multicolumn{1}{c|}{\multirow{1}{*}{0.25 (0.01)}}
    &\multicolumn{1}{c|}{\multirow{1}{*}{0.25 (0.01)}}
    &\multicolumn{1}{c|}{\multirow{1}{*}{0.25 (0.00)}}
    &\multicolumn{1}{c|}{\multirow{1}{*}{0.26 (0.01)}}
    % NOVI
    &\multicolumn{1}{c|}{\multirow{1}{*}{\textbf{0.24 (0.00)}}}
    &\multicolumn{1}{c|}{\multirow{1}{*}{0.28 (0.00)}}
    &\multicolumn{1}{c|}{\multirow{1}{*}{0.26 (0.00)}}
    &\multicolumn{1}{c|}{\multirow{1}{*}{0.27 (0.00)}}\\
    \hline
    
    \end{tabular}}
    \caption{Tabular version of Figure 2 in the main text.}
    \label{tab:regression-tabular}
\end{table*}

\begin{table*}
    \resizebox{\textwidth}{!}{
    \begin{tabular}{|c|c|c|c|c|c|c|c|c|c|c|c|c|}
    \hline
    \multicolumn{1}{|c|}{\multirow{2}{*}{\textbf{Method}}} &\multicolumn{4}{c|}{\multirow{1}{*}{\textbf{Estate}}} &\multicolumn{4}{c|}{\multirow{1}{*}{\textbf{Elevators}}} \\
    \cline{2-9}
    
    \multicolumn{1}{|c|}{} &\multicolumn{1}{c|}{\multirow{1}{*}{L=2}} &\multicolumn{1}{c|}{\multirow{1}{*}{L=3}} &\multicolumn{1}{c|}{\multirow{1}{*}{L=4}} &\multicolumn{1}{c|}{\multirow{1}{*}{L=5}}
    &\multicolumn{1}{c|}{\multirow{1}{*}{L=2}}
    &\multicolumn{1}{c|}{\multirow{1}{*}{L=3}}
    &\multicolumn{1}{c|}{\multirow{1}{*}{L=4}}
    &\multicolumn{1}{c|}{\multirow{1}{*}{L=5}} \\
    \hline
    
    % DSVI
    \multicolumn{1}{|c|}{\multirow{1}{*}{\textbf{DSVI}}}
    % Estate
    &\multicolumn{1}{c|}{\multirow{1}{*}{0.65 (0.02)}} &\multicolumn{1}{c|}{\multirow{1}{*}{0.66 (0.02)}} &\multicolumn{1}{c|}{\multirow{1}{*}{0.50 (0.02)}} &\multicolumn{1}{c|}{\multirow{1}{*}{0.64 (0.02)}}
    % Elevators
    &\multicolumn{1}{c|}{\multirow{1}{*}{0.37 (0.00)}} &\multicolumn{1}{c|}{\multirow{1}{*}{0.36 (0.00)}} &\multicolumn{1}{c|}{\multirow{1}{*}{0.37 (0.00)}} &\multicolumn{1}{c|}{\multirow{1}{*}{0.36 (0.00)}} \\
    \hline
    
    % SGHMC
    \multicolumn{1}{|c|}{\multirow{1}{*}{\textbf{SGHMC}}}
    % Estate
    &\multicolumn{1}{c|}{\multirow{1}{*}{0.54 (0.01)}} &\multicolumn{1}{c|}{\multirow{1}{*}{0.50 (0.01)}} &\multicolumn{1}{c|}{\multirow{1}{*}{0.53 (0.01)}} &\multicolumn{1}{c|}{\multirow{1}{*}{0.61 (0.01)}}
    % Elevators
    &\multicolumn{1}{c|}{\multirow{1}{*}{0.36 (0.00)}} &\multicolumn{1}{c|}{\multirow{1}{*}{0.36 (0.00)}} &\multicolumn{1}{c|}{\multirow{1}{*}{\textbf{0.35 (0.00)}}} &\multicolumn{1}{c|}{\multirow{1}{*}{\textbf{0.35 (0.00)}}} \\ 
    \hline
    
    % NOVI
    \multicolumn{1}{|c|}{\multirow{1}{*}{\textbf{NOVI}}}
    % Estate
    &\multicolumn{1}{c|}{\multirow{1}{*}{0.56 (0.02)}} &\multicolumn{1}{c|}{\multirow{1}{*}{0.40 (0.02)}} &\multicolumn{1}{c|}{\multirow{1}{*}{0.40 (0.01)}} &\multicolumn{1}{c|}{\multirow{1}{*}{\textbf{0.39 (0.02)}}}
    % Elevators
    &\multicolumn{1}{c|}{\multirow{1}{*}{0.36 (0.00)}} &\multicolumn{1}{c|}{\multirow{1}{*}{\textbf{0.35 (0.00)}}} &\multicolumn{1}{c|}{\multirow{1}{*}{\textbf{0.35 (0.00)}}} &\multicolumn{1}{c|}{\multirow{1}{*}{\textbf{0.35 (0.00)}}} \\
    \hline
    
    \end{tabular}}
    \caption{Additional experiments for real-world datasets. It shows regression mean test RMSE values with its standard deviation on the round bracket. L denotes the number of layers in DGP models.}
    \label{tab:additional-tabular}
\end{table*}

\subsection{Ablation Study on Classification Datasets}
\label{ablation-classification}
We also performed ablation study on clasification datasets and reported its results by test accuracy in Figure \ref{fig:ablation-classification}. From which it can be seen that NOVI not only achieves better results on large scale datasets, which demonstrates its scalability, but also the results on the test set have far exceeded the performance of the Monte Carlo log-likelihood maximization method on the training set, suggesting the feasibility of adversarial training.
\subsection{Tabular version of Figure \ref{fig:regression} in the Main Text}
%\ref{fig:regression}
\label{tabular-regression}
Tabular version of Figure \ref{fig:regression} in the main text can be seen in Table \ref{tab:regression-tabular}.

\begin{table}
    \centering
    \resizebox{\linewidth}{!}{\begin{tabular}{|c|c|c|c|c|c|}
    \hline
    \multicolumn{2}{|c|}{\multirow{1}{*}{}} &\multicolumn{1}{c|}{\multirow{1}{*}{\textbf{Concrete}}} &\multicolumn{1}{c|}{\multirow{1}{*}{\textbf{Energy}}}
    &\multicolumn{1}{c|}{\multirow{1}{*}{\textbf{Boston}}}
    &\multicolumn{1}{c|}{\multirow{1}{*}{\textbf{Kin8nm}}} \\
    \hline
    
    \multicolumn{2}{|c|}{\multirow{1}{*}{\textbf{Iteration}}} &\multicolumn{1}{c|}{\multirow{1}{*}{500}} 
    &\multicolumn{1}{c|}{\multirow{1}{*}{600}} 
    &\multicolumn{1}{c|}{\multirow{1}{*}{300}} 
    &\multicolumn{1}{c|}{\multirow{1}{*}{500}} \\
    \hline
    
    \multicolumn{2}{|c|}{\multirow{1}{*}{\textbf{RMSE (M=50)}}} &\multicolumn{1}{c|}{\multirow{1}{*}{0.28 (0.00)}} 
    &\multicolumn{1}{c|}{\multirow{1}{*}{0.04 (0.00)}} 
    &\multicolumn{1}{c|}{\multirow{1}{*}{0.23 (0.00)}} 
    &\multicolumn{1}{c|}{\multirow{1}{*}{0.26 (0.00)}} \\
    \hline
    
    \multicolumn{2}{|c|}{\multirow{1}{*}{\textbf{Time (M=50)}}} &\multicolumn{1}{c|}{\multirow{1}{*}{0.397s}} 
    &\multicolumn{1}{c|}{\multirow{1}{*}{0.404s}} 
    &\multicolumn{1}{c|}{\multirow{1}{*}{0.380s}} 
    &\multicolumn{1}{c|}{\multirow{1}{*}{0.600s}} \\
    \hline
    
    \multicolumn{2}{|c|}{\multirow{1}{*}{\textbf{RMSE (M=100)}}} &\multicolumn{1}{c|}{\multirow{1}{*}{0.24 (0.00)}} 
    &\multicolumn{1}{c|}{\multirow{1}{*}{0.04 (0.00)}} 
    &\multicolumn{1}{c|}{\multirow{1}{*}{0.20 (0.01)}} 
    &\multicolumn{1}{c|}{\multirow{1}{*}{0.24 (0.00)}} \\
    \hline
    
    \multicolumn{2}{|c|}{\multirow{1}{*}{\textbf{Time (M=100)}}} &\multicolumn{1}{c|}{\multirow{1}{*}{0.403s}} 
    &\multicolumn{1}{c|}{\multirow{1}{*}{0.420s}} 
    &\multicolumn{1}{c|}{\multirow{1}{*}{0.400s}} 
    &\multicolumn{1}{c|}{\multirow{1}{*}{0.613s}} \\
    \hline
    
    \multicolumn{2}{|c|}{\multirow{1}{*}{\textbf{RMSE (M=200)}}} &\multicolumn{1}{c|}{\multirow{1}{*}{0.20 (0.00)}} 
    &\multicolumn{1}{c|}{\multirow{1}{*}{0.03 (0.00)}} 
    &\multicolumn{1}{c|}{\multirow{1}{*}{0.20 (0.00)}} 
    &\multicolumn{1}{c|}{\multirow{1}{*}{0.24 (0.00)}} \\
    \hline
    
    \multicolumn{2}{|c|}{\multirow{1}{*}{\textbf{Time (M=200)}}} &\multicolumn{1}{c|}{\multirow{1}{*}{0.408s}} 
    &\multicolumn{1}{c|}{\multirow{1}{*}{0.450s}} 
    &\multicolumn{1}{c|}{\multirow{1}{*}{0.410s}} 
    &\multicolumn{1}{c|}{\multirow{1}{*}{0.646s}} \\
    \hline
    
    \multicolumn{2}{|c|}{\multirow{1}{*}{\textbf{RMSE (M=400)}}} &\multicolumn{1}{c|}{\multirow{1}{*}{0.19 (0.00)}} 
    &\multicolumn{1}{c|}{\multirow{1}{*}{0.03 (0.00)}} 
    &\multicolumn{1}{c|}{\multirow{1}{*}{0.18 (0.01)}} 
    &\multicolumn{1}{c|}{\multirow{1}{*}{0.23 (0.00)}} \\
    \hline
    
    \multicolumn{2}{|c|}{\multirow{1}{*}{\textbf{Time (M=400)}}} &\multicolumn{1}{c|}{\multirow{1}{*}{0.408s}} 
    &\multicolumn{1}{c|}{\multirow{1}{*}{0.450s}} 
    &\multicolumn{1}{c|}{\multirow{1}{*}{0.420s}} 
    &\multicolumn{1}{c|}{\multirow{1}{*}{0.658s}} \\
    \hline
    
    \end{tabular}}
    \caption{Comparison of number of inducing points ($50$, $100$, $200$ and $400$) using $2$-layer DGP model on $4$ UCI regression datasets. M denotes the number of inducing points per layer.}
    \label{tab:inducing}
\end{table}

\subsection{Comparison about Inducing Points}
\label{inducing-comparison}
In order to investigate the robustness of NOVI at different numbers of induced points, we have performed ablation study to compare accuracy and training time on $4$ UCI regression datasets using $2$-layer DGP model. For each dataset, number of iteration is set to be the same for fair comparison. Results are shown in Table \ref{tab:inducing}. From which it can be seen that the performance increases gradually with the number of induction points, while the time fluctuates only slightly, which shows the robustness of NOVI to the number of inducing points.

\subsection{Additional Experiments}
\label{additional-experiments}
We have performed additional regression experiments for two real-world datasets: Estate and Elevators. Results are shown in Table \ref{tab:additional-tabular}. From these two datasets, it can be seen that NOVI has achieved better RMSE value than other two methods.

In order to further demonstrate the advantages of our method compared to the state-of-the-art  approaches, we conducted comparisons with the two most recent methods for DGP posterior inference, IWVI \cite{salimbeni2019deep} and its variant (IWVI with DREG estimators) \cite{rudner2021signal}. Introducing importance weighting for posterior sampling in IWVI not only enhances the variational lower bound but also serves as a crucial variation of DSVI. In contrast, our proposed method, as explained in the text, distinguishes itself from variational approaches based on KL divergence. We presented their results on 8 UCI datasets in Table \ref{tab:IWVI}. The results show that NOVI still outperforms the latest DGP posterior inference methods.

\begin{table*}[t]
    %\centering
    \resizebox{\textwidth}{!}{
    \begin{tabular}{|c|c|c|c|c|c|c|c|c|c|c|c|c|}
    \hline
    \multicolumn{1}{|c|}{\multirow{1}{*}{\textbf{Data}}} &\multicolumn{1}{c|}{\multirow{1}{*}{\textbf{IWVI 2}}} &\multicolumn{1}{c|}{\multirow{1}{*}{\textbf{IWVI 3}}}
    &\multicolumn{1}{c|}{\multirow{1}{*}{\textbf{IWVI 4}}}
    &\multicolumn{1}{c|}{\multirow{1}{*}{\textbf{IWVI 5}}}
    &\multicolumn{1}{c|}{\multirow{1}{*}{\textbf{IWVI-DREG 2}}}
    &\multicolumn{1}{c|}{\multirow{1}{*}{\textbf{IWVI-DREG 3}}}
    &\multicolumn{1}{c|}{\multirow{1}{*}{\textbf{IWVI-DREG4}}}
    &\multicolumn{1}{c|}{\multirow{1}{*}{\textbf{IWVI-DREG 5}}}
  
    &\multicolumn{1}{c|}{\multirow{1}{*}{\textbf{NOVI 2}}}
    &\multicolumn{1}{c|}{\multirow{1}{*}{\textbf{NOVI 3}}}
    &\multicolumn{1}{c|}{\multirow{1}{*}{\textbf{NOVI 4}}}
    &\multicolumn{1}{c|}{\multirow{1}{*}{\textbf{NOVI 5}}}\\
    \hline
    
    % Boston
    \multicolumn{1}{|c|}{\multirow{1}{*}{\textbf{Boston}}}
    % IWVI
    &\multicolumn{1}{c|}{\multirow{1}{*}{0.33 (0.02)}} 
    &\multicolumn{1}{c|}{\multirow{1}{*}{0.35 (0.02)}} 
    &\multicolumn{1}{c|}{\multirow{1}{*}{0.35 (0.02)}} 
    &\multicolumn{1}{c|}{\multirow{1}{*}{0.36 (0.02)}}
    % SGHMC
    &\multicolumn{1}{c|}{\multirow{1}{*}{0.32 (0.02)}}
    &\multicolumn{1}{c|}{\multirow{1}{*}{0.33 (0.02)}}
    &\multicolumn{1}{c|}{\multirow{1}{*}{0.35 (0.02)}}
    &\multicolumn{1}{c|}{\multirow{1}{*}{0.36 (0.02)}}
 
    % NOVI
    &\multicolumn{1}{c|}{\multirow{1}{*}{\textbf{0.20 (0.01)}}}
    &\multicolumn{1}{c|}{\multirow{1}{*}{0.34 (0.02)}}
    &\multicolumn{1}{c|}{\multirow{1}{*}{0.40 (0.03)}}
    &\multicolumn{1}{c|}{\multirow{1}{*}{0.38 (0.02)}}\\
    \hline
    
    % Energy
    \multicolumn{1}{|c|}{\multirow{1}{*}{\textbf{Energy}}}
    % IWVI
    &\multicolumn{1}{c|}{\multirow{1}{*}{0.05 (0.00)}} 
    &\multicolumn{1}{c|}{\multirow{1}{*}{0.06 (0.00)}} 
    &\multicolumn{1}{c|}{\multirow{1}{*}{0.06 (0.00)}} 
    &\multicolumn{1}{c|}{\multirow{1}{*}{0.06 (0.00)}}
    % SGHMC
    &\multicolumn{1}{c|}{\multirow{1}{*}{0.05 (0.00)}}
    &\multicolumn{1}{c|}{\multirow{1}{*}{0.06 (0.00)}}
    &\multicolumn{1}{c|}{\multirow{1}{*}{0.06 (0.00)}}
    &\multicolumn{1}{c|}{\multirow{1}{*}{0.06 (0.00)}}
 
    % NOVI
    &\multicolumn{1}{c|}{\multirow{1}{*}{\textbf{0.04 (0.00)}}}
    &\multicolumn{1}{c|}{\multirow{1}{*}{0.05 (0.00)}}
    &\multicolumn{1}{c|}{\multirow{1}{*}{0.06 (0.00)}}
    &\multicolumn{1}{c|}{\multirow{1}{*}{0.06 (0.00)}}\\
    \hline
    
    % Power
    \multicolumn{1}{|c|}{\multirow{1}{*}{\textbf{Power}}}
    % IWVI
    &\multicolumn{1}{c|}{\multirow{1}{*}{0.22 (0.00)}} 
    &\multicolumn{1}{c|}{\multirow{1}{*}{0.22 (0.00)}} 
    &\multicolumn{1}{c|}{\multirow{1}{*}{0.22 (0.00)}} 
    &\multicolumn{1}{c|}{\multirow{1}{*}{0.22 (0.00)}}
    % IWVI
    &\multicolumn{1}{c|}{\multirow{1}{*}{0.22 (0.00)}}
    &\multicolumn{1}{c|}{\multirow{1}{*}{0.22 (0.00)}}
    &\multicolumn{1}{c|}{\multirow{1}{*}{0.22 (0.00)}}
    &\multicolumn{1}{c|}{\multirow{1}{*}{0.22 (0.00)}}
  
    % NOVI
    &\multicolumn{1}{c|}{\multirow{1}{*}{0.22 (0.00)}}
    &\multicolumn{1}{c|}{\multirow{1}{*}{\textbf{0.21 (0.00)}}}
    &\multicolumn{1}{c|}{\multirow{1}{*}{\textbf{0.21 (0.00)}}}
    &\multicolumn{1}{c|}{\multirow{1}{*}{\textbf{0.21 (0.00)}}}\\
    \hline
    
    % Concrete
    \multicolumn{1}{|c|}{\multirow{1}{*}{\textbf{Concrete}}}
    % DSVI
    &\multicolumn{1}{c|}{\multirow{1}{*}{0.32 (0.01)}} 
    &\multicolumn{1}{c|}{\multirow{1}{*}{0.28 (0.01)}} 
    &\multicolumn{1}{c|}{\multirow{1}{*}{0.27 (0.01)}} 
    &\multicolumn{1}{c|}{\multirow{1}{*}{0.27 (0.01)}}
    % SGHMC
    &\multicolumn{1}{c|}{\multirow{1}{*}{0.31 (0.01)}}
    &\multicolumn{1}{c|}{\multirow{1}{*}{0.27 (0.01)}}
    &\multicolumn{1}{c|}{\multirow{1}{*}{0.27 (0.01)}}
    &\multicolumn{1}{c|}{\multirow{1}{*}{0.26 (0.01)}}
  
    % NOVI
    &\multicolumn{1}{c|}{\multirow{1}{*}{0.24 (0.00)}}
    &\multicolumn{1}{c|}{\multirow{1}{*}{0.25 (0.00)}}
    &\multicolumn{1}{c|}{\multirow{1}{*}{0.24 (0.00)}}
    &\multicolumn{1}{c|}{\multirow{1}{*}{\textbf{0.23 (0.00)}}}\\
    \hline
    
    % Yacht
    \multicolumn{1}{|c|}{\multirow{1}{*}{\textbf{Yacht}}}
    % DSVI
    &\multicolumn{1}{c|}{\multirow{1}{*}{0.07 (0.00)}} 
    &\multicolumn{1}{c|}{\multirow{1}{*}{0.09 (0.00)}} 
    &\multicolumn{1}{c|}{\multirow{1}{*}{0.07 (0.00)}} 
    &\multicolumn{1}{c|}{\multirow{1}{*}{0.06 (0.00)}}
    % SGHMC
    &\multicolumn{1}{c|}{\multirow{1}{*}{0.07 (0.00)}}
    &\multicolumn{1}{c|}{\multirow{1}{*}{0.08 (0.00)}}
    &\multicolumn{1}{c|}{\multirow{1}{*}{0.07 (0.0)}}
    &\multicolumn{1}{c|}{\multirow{1}{*}{0.06 (0.00)}}

    % NOVI
    &\multicolumn{1}{c|}{\multirow{1}{*}{\textbf{0.03} (0.00)}}
    &\multicolumn{1}{c|}{\multirow{1}{*}{0.09 (0.01)}}
    &\multicolumn{1}{c|}{\multirow{1}{*}{0.08 (0.00)}}
    &\multicolumn{1}{c|}{\multirow{1}{*}{0.06 (0.00)}}\\
    \hline
    % Qsar
    \multicolumn{1}{|c|}{\multirow{1}{*}{\textbf{Qsar}}}
    % IWVI
    &\multicolumn{1}{c|}{\multirow{1}{*}{0.54 (0.00)}} 
    &\multicolumn{1}{c|}{\multirow{1}{*}{0.47 (0.00)}} 
    &\multicolumn{1}{c|}{\multirow{1}{*}{0.44 (0.00)}} 
    &\multicolumn{1}{c|}{\multirow{1}{*}{0.43 (0.00)}}
    % SGHMC
    &\multicolumn{1}{c|}{\multirow{1}{*}{0.52 (0.00)}}
    &\multicolumn{1}{c|}{\multirow{1}{*}{0.46 (0.00)}}
    &\multicolumn{1}{c|}{\multirow{1}{*}{0.43 (0.00)}}
    &\multicolumn{1}{c|}{\multirow{1}{*}{\textbf{0.42 (0.00)}}}

    % NOVI
    &\multicolumn{1}{c|}{\multirow{1}{*}{0.51 (0.00)}}
    &\multicolumn{1}{c|}{\multirow{1}{*}{0.46 (0.01)}}
    &\multicolumn{1}{c|}{\multirow{1}{*}{0.45 (0.01)}}
    &\multicolumn{1}{c|}{\multirow{1}{*}{0.44 (0.01)}}\\
    \hline
    
    % Protein
    \multicolumn{1}{|c|}{\multirow{1}{*}{\textbf{Protein}}}
    % DSVI
    &\multicolumn{1}{c|}{\multirow{1}{*}{0.72 (0.00)}} 
    &\multicolumn{1}{c|}{\multirow{1}{*}{0.67 (0.00)}} 
    &\multicolumn{1}{c|}{\multirow{1}{*}{0.67 (0.00)}} 
    &\multicolumn{1}{c|}{\multirow{1}{*}{0.67 (0.00)}}
    % SGHMC
    &\multicolumn{1}{c|}{\multirow{1}{*}{0.70 (0.01)}}
    &\multicolumn{1}{c|}{\multirow{1}{*}{0.66 (0.01)}}
    &\multicolumn{1}{c|}{\multirow{1}{*}{0.66 (0.01)}}
    &\multicolumn{1}{c|}{\multirow{1}{*}{0.66 (0.00)}}

    % NOVI
    &\multicolumn{1}{c|}{\multirow{1}{*}{0.67 (0.00)}}
    &\multicolumn{1}{c|}{\multirow{1}{*}{\textbf{0.65 (0.00)}}}
    &\multicolumn{1}{c|}{\multirow{1}{*}{0.66 (0.00)}}
    &\multicolumn{1}{c|}{\multirow{1}{*}{0.66 (0.00)}}\\
    \hline
    
    % Kin8nm
    \multicolumn{1}{|c|}{\multirow{1}{*}{\textbf{Kin8nm}}}
    % DSVI
    &\multicolumn{1}{c|}{\multirow{1}{*}{0.37 (0.00)}} 
    &\multicolumn{1}{c|}{\multirow{1}{*}{0.34 (0.00)}} 
    &\multicolumn{1}{c|}{\multirow{1}{*}{0.31 (0.00)}} 
    &\multicolumn{1}{c|}{\multirow{1}{*}{0.29 (0.00)}}
    % SGHMC
    &\multicolumn{1}{c|}{\multirow{1}{*}{0.35 (0.00)}}
    &\multicolumn{1}{c|}{\multirow{1}{*}{0.32 (0.00)}}
    &\multicolumn{1}{c|}{\multirow{1}{*}{0.31 (0.00)}}
    &\multicolumn{1}{c|}{\multirow{1}{*}{0.29 (0.00)}}
   
    % NOVI
    &\multicolumn{1}{c|}{\multirow{1}{*}{\textbf{0.24 (0.00)}}}
    &\multicolumn{1}{c|}{\multirow{1}{*}{0.28 (0.00)}}
    &\multicolumn{1}{c|}{\multirow{1}{*}{0.26 (0.00)}}
    &\multicolumn{1}{c|}{\multirow{1}{*}{0.27 (0.00)}}\\
    \hline
    
    \end{tabular}}
    \caption{comparisons with the two most recent methods; IWVI and IWVI with DREG estimators. the results of UCI test RMSE are reported   }
    \label{tab:IWVI}
\end{table*}

\section{Training Details}
\label{training-details}

\subsection{UCI Datasets}
\label{uci-details}
\paragraph{Training} We conducted a random $0.9 / 0.1$ train/test split and normalized the features to the range $[-1, 1]$. The depth $L$ of DGP models varied from $2$ to $5$, with $100$ inducing points per layer, which are initialized by sampling from isotrophic Gaussian distribution. The output dimension for each hidden layer is set to $1$ for final layer and $10$ for others. We have utilized RQ kernel for all tasks. For all datasets, we have optimized hyper-parameters and network parameters jointly and utilized different learning rate, $0.02$ for hyper-parameters and $0.001$ for network parameters using Adam optimizer \cite{kingma2014adam}. The dimension of noise $\boldsymbol{\epsilon}$ used to generate $\mathcal{U}$ is set to $200$ for all datasets. We train for almost $500$ iterations for all datasets. DSVI and SGHMC methods are initialized the same as NOVI to obtain a fair comparison.

\paragraph{Network Settings} In this study, the selection of the generator and discriminator networks is done manually. However, to further optimize the hyperparameters of these neural networks, we propose a classical grid search approach for each experimental dataset. 
Taking the energy dataset as an example, we set the generator and discriminator networks as three-layer neural networks. To explore the optimal choices for the activation function, we present the results in Table \ref{tab:net-hyper}. From the results in Table \ref{tab:net-hyper}, we can identify the relatively superior combination from the alternative choices of generator and discriminator networks. 
For the other hyperparameters of the neural networks, such as the number of hidden units in the intermediate layers, we can also fine-tune them using the same grid search method. 
As for further improvements or refinements of the algorithm, we leave it as future work. 
By adopting this systematic grid search method, we can effectively optimize the hyperparameters of the generator and discriminator networks, leading to improved performance on the specific dataset under consideration. This approach offers a structured framework for selecting and fine-tuning the neural network hyperparameters in the context of our study.

\begin{table}[t]
    \centering
    \resizebox{\linewidth}{!}{
    \begin{tabular}{|c|c|c|c|}
    \hline
    \diagbox{Discriminator}{Generator}    & Tanh & Prelu  & Sigmoid \\
    \hline
    Tanh & 0.041(0.001) & 0.042(0.001)  &  0.039(0.003)\\
    \hline
    Prelu & 0.041(0.001) & 0.038(0.001) &  0.040 (0.003) \\
    \hline
    Sigmoid & 0.038(0.001) & 0.034(0.003) &  0.040(0.002) \\
    \hline
    \end{tabular}}
    \caption{An example: When both the generator and discriminator are three-layer neural networks, and their activation functions are respectively taken as Tanh, Prelu, and Sigmoid, the RMSE of the NOVI algorithm on the energy dataset is reported.}
    \label{tab:net-hyper}
\end{table}
\paragraph{Regularization Strategies}

Based on our experimental results on the energy dataset, we tested the effect of different $\lambda$ values on the model performance. We took $\lambda$ values in increments of 10, namely 1, 10, 100, 1000, and 10000,  maintained the other hyperparameters unchanged and recorded the experimental results (RMSE on both the training and testing sets) in Table \ref{tab:lambda}. From Table\ref{tab:lambda}, it can be observed that the model performs the best when lambda is set to 10. Larger lambda values may lead to over-regularization, thereby reducing the model performance on the training set. On the other hand, too small lambda values may result in overfitting, although the model performs better on the training set, it fails to generalize on the testing set. Therefore, it is recommended to select an appropriate lambda value to adjust the regularization degree of the model. By conducting experiments on different $\lambda$ values, we can find the optimal lambda value that achieves good performance on both the training and testing sets.

In order to investigate the impact of using the Hutchinson estimator on NOVI, as proposed in our main text to significantly reduce the computational complexity of calculating the trace of the Jacobian matrix, we conducted an ablation experiment on Energy dataset and a considerably large Elevators dataset consisting of over ten thousand data points and 18-dimensional features. We compared the direct computation of the Jacobian matrix with the iterative approach using the Hutchinson estimator and reported the corresponding time consumption in Table 1. From the results presented in the table, it can be observed that the Hutchinson estimator significantly reduces the computational complexity of the NOVI method, which aligns with the theoretical predictions.
\begin{table}[t]
    \centering
    \resizebox{\linewidth}{!}{
    \begin{tabular}{|c|c|c|c|c|c|}
    \hline
    \diagbox{Dataset}{$\lambda$}    & 1 & 10  & 100& 1000& 10000 \\
    \hline
    Train & 0.025(0.001) & 0.031(0.001)  &  0.032(0.001)&0.042(0.001)&0.050(0.001)\\
    \hline
    Test & 0.050(0.001) & 0.037(0.001) &  0.038 (0.003)&  0.043 (0.003)&  0.050 (0.003) \\
    \hline
    \end{tabular}}
    \caption{An example: The RMSE values of the training and testing sets corresponding to different $\lambda$ values on the energy dataset.}
    \label{tab:lambda}
\end{table}

\begin{table}[t]
    \centering
    \resizebox{\linewidth}{!}{
    \begin{tabular}{|c|c|c|}
    \hline
          & Hutchinson Estimator &  Direct Computation  \\
    \hline
    Energy   & 0.391s/iter & 0.454s/iter\\
    
    \hline
    Elevators & 0.492s/iter & 0.587s/iter  \\
    \hline
    
    \end{tabular}}
    \caption{  the impact of utilizing the Hutchinson estimator for calculating the trace of the Jacobian matrix on the computational complexity.}
    \label{tab:hum}
\end{table}

\subsection{Image Datasets}
\label{image-details}
\paragraph{Training} We have followed the division of the original dataset and normalized pixel values to $[-1, 1]$. The depth $L$ of DGP models are varied from $3$ to $4$ with $100$ inducing points per layer, which are initialized by sampling from isotrophic Gaussian distribution. The output dimension for each hidden layer is set to be $10$ for final layer (which is the exact number of class to predict), and $60$ for others. We have utilized RQ kernel for all tasks. For all datasets, we have optimized hyper-parameters and network parameters jointly and utilized different learning rate, $0.02$ for hyper-parameters and $0.001$ for network parameters using Adam optimizer \cite{kingma2014adam}. The dimension of noise $\boldsymbol{\epsilon}$ used to generate $\mathcal{U}$ is set to $200$ for all datasets. We train for almost $10k$ iterations for all datasets. DSVI and SGHMC methods are initialized the same as NOVI to obtain a fair comparison.

\paragraph{Network Settings}  The selection of the
generator and discriminator networks is done manually, we also use a classical grid search approach for each experimental dataset.

%\begin{small}
    %\bibliographystyle{unsrt}
    %\bibliography{Reference}
%\end{small}

%%%%%%%%%%%%%%%%%%%%%%%%%%%%%%%%%%%%%%%%%%%%%%%%%%%%%%%%%%%%%%%%%%%%%%%%%%%%%%%
%%%%%%%%%%%%%%%%%%%%%%%%%%%%%%%%%%%%%%%%%%%%%%%%%%%%%%%%%%%%%%%%%%%%%%%%%%%%%%%
\bibliography{NOVI}
\bibliographystyle{IEEEtran}

\begin{IEEEbiography}[{\includegraphics[width=1in,height=1.25in,clip,keepaspectratio]{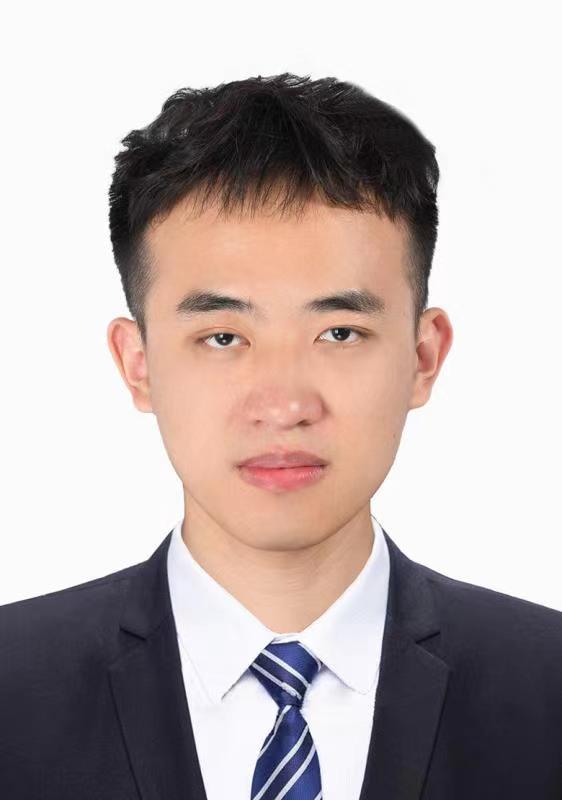}}]{Jian Xu} received the B.S. degree in the Department of Business Administration at the Communication University of China, Beijing, China in 2017. Currently, he is pursuing a Ph.D. at South China University of Technology, focusing on machine learning, stochastic processes, generative models, and their applications.
\end{IEEEbiography}

\begin{IEEEbiography}[{\includegraphics[width=1in,height=1.25in,clip,keepaspectratio]{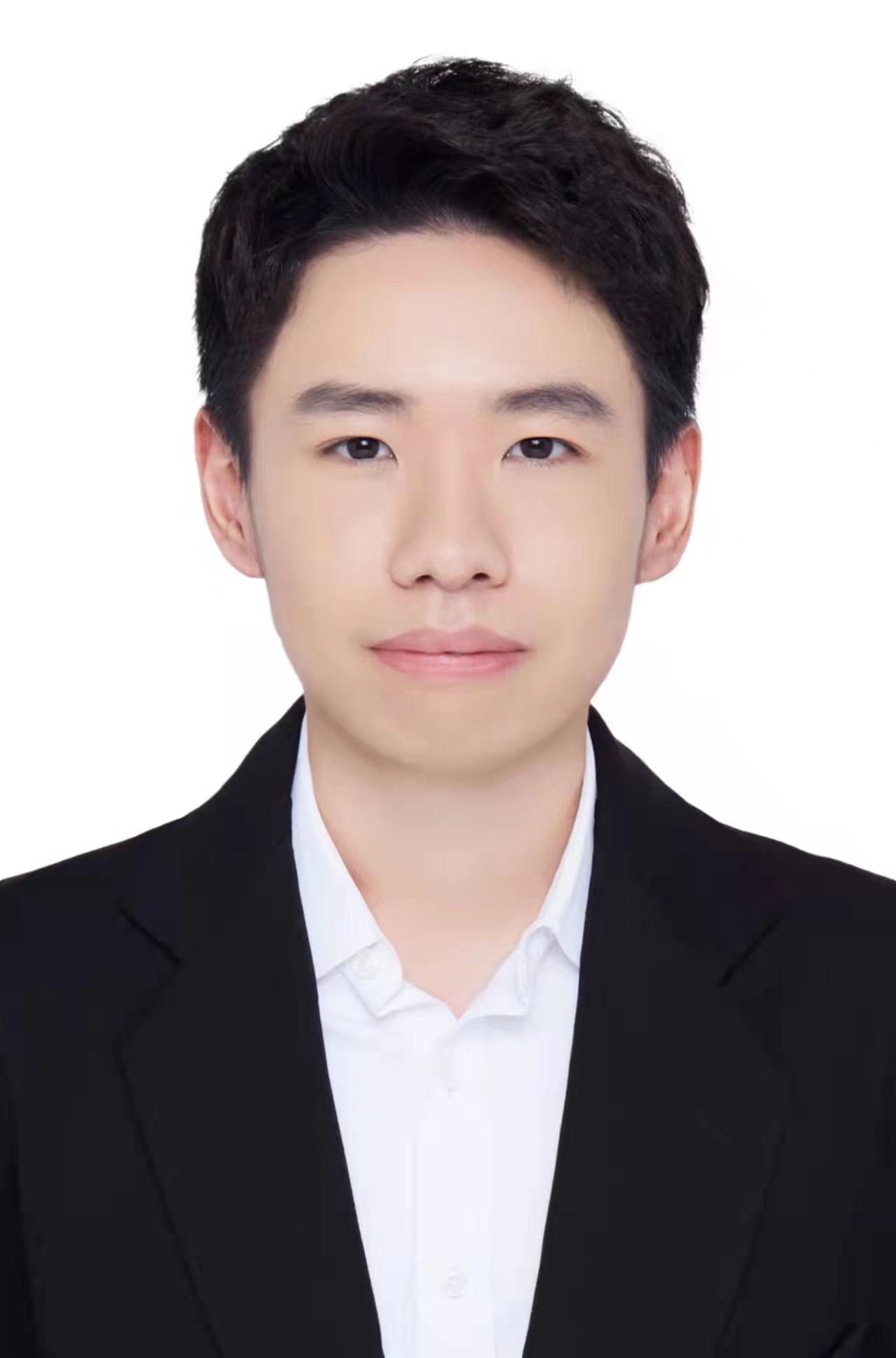}}]{Shian Du}
received the B.S. degree in applied mathematics from South China University of Technology, Guangzhou, China in 2023. He is currently a Master Student in Tsinghua University. His research interests include large-scale text-to-video and image-to-video generation. 
\end{IEEEbiography}

\begin{IEEEbiography}[{\includegraphics[width=1in,height=1.25in,clip,keepaspectratio]{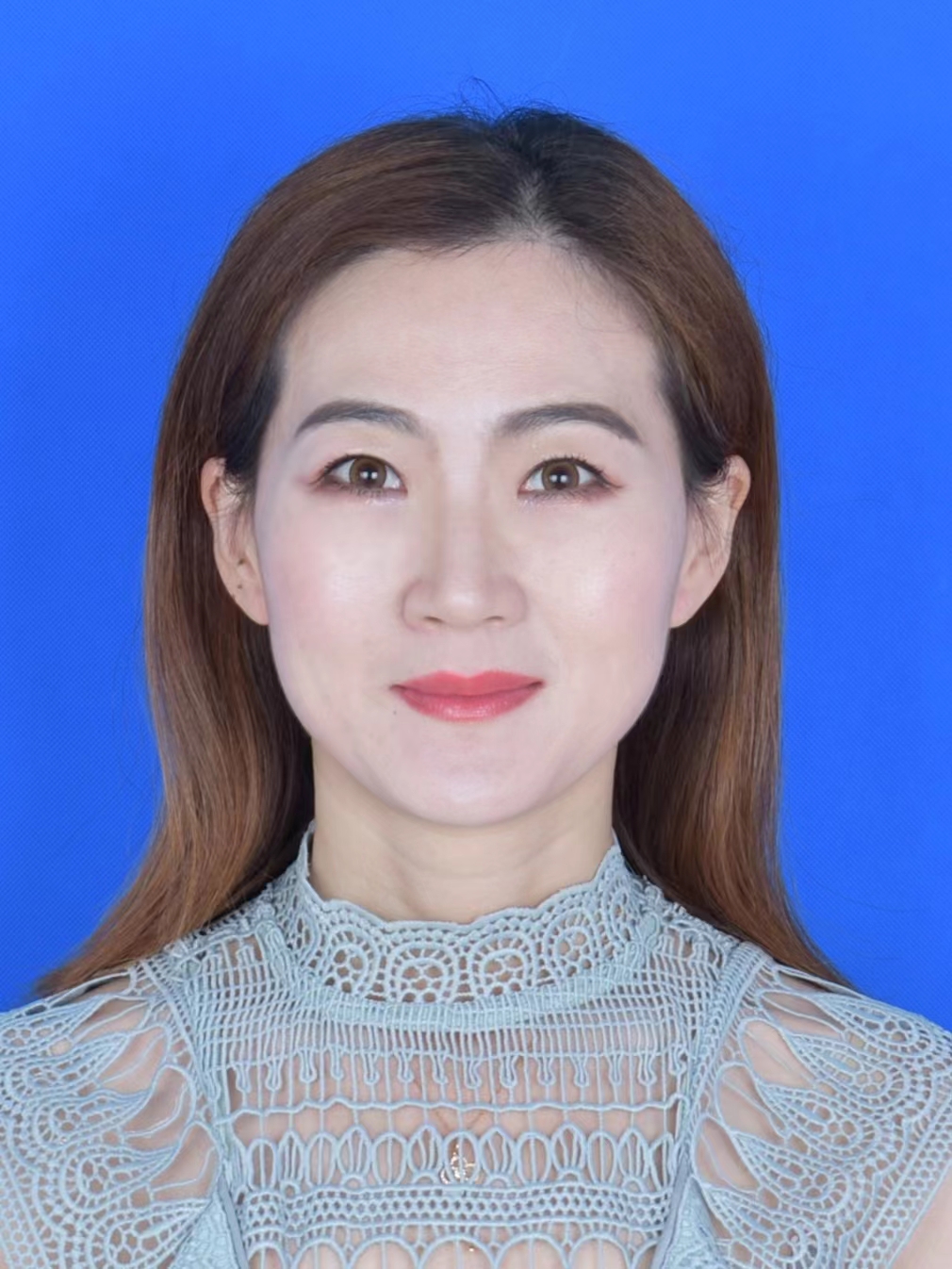}}]{Junmei Yang}
received the M.S. degree in cybernetics from the Chinese Academy of Sciences, Beijing,
China, in 2005, and the Ph.D. degree in systems
science from the Graduate School of Informatics,
Kyoto University, Kyoto, Japan, in 2008. She is currently an Associate Professor with the South China
University of Technology, Guangzhou, China. Her
current research interests focus on image processing,
speech enhancement, machine learning and artificial
intelligence
\end{IEEEbiography}
\begin{IEEEbiography}[{\includegraphics[width=1in,height=1.25in,clip,keepaspectratio]{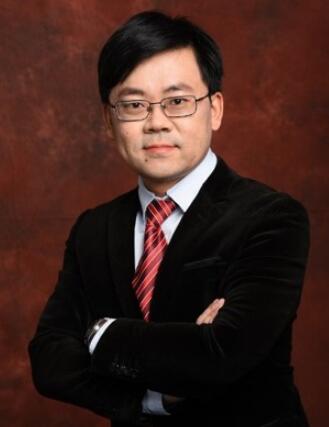}}]{Qianli Ma}
 (Member, IEEE) received the Ph.D. degree in
computer science from the South China University of Technology, Guangzhou, China, in 2008. He is a Professor with the School of Computer Science and Engineering, South China University of Technology. From 2016 to 2017, he was a Visiting Scholar with the University of California, San Diego.
His current research interests include machine
learning algorithms, data-mining methodologies,
and their applications.
\end{IEEEbiography}

\begin{IEEEbiography}[{\includegraphics[width=1in,height=1.25in,clip,keepaspectratio]{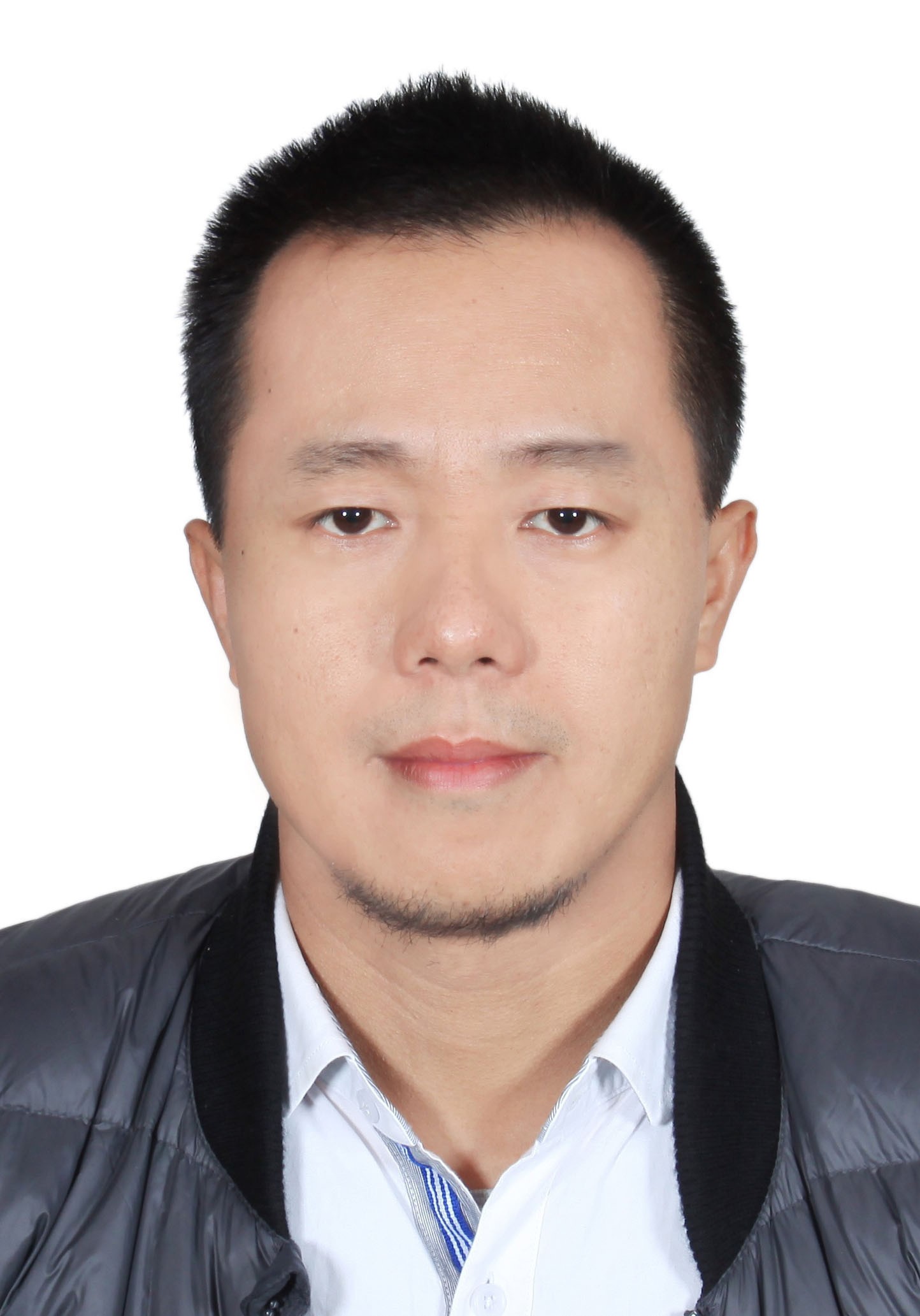}}]{Delu Zeng}
 (Member, IEEE) received his Bachelor degree in applied
mathematics and his Phd degree in information
and signal processing in South China University of
Technology (SCUT) in June 2003 and June 2010,
respectively. He is currently a full professor with the
School of Electronic and Information Engineering in
South China University of Technology (SCUT) in
Guangzhou.  He was a Visiting Scholar with Columbia
University, University of
Oulu and University of Waterloo. His current research focuses
on applied mathematics and its interdisciplinary application, including statistics
learning, image and speech processing, computational intelligence, machine learning, fitting, and
approximation and their applications to communication, industrial intelligence, etc. 
\end{IEEEbiography}

\end{document}